\documentclass[11pt,onecolumn]{article}
\author{
 Zirui Yan \qquad Arpan Mukherjee \qquad Burak Var\i c\i \qquad  Ali Tajer
 \thanks{The authors are with the Department of Electrical, Computer, and Systems Engineering, Rensselaer Polytechnic Institute, Troy, NY 12180.}
 }

\date{}

\usepackage[utf8]{inputenc} 
\usepackage[T1]{fontenc}    
\usepackage{hyperref}       
\usepackage{url}            
\usepackage{booktabs}       
\usepackage{amsfonts}       
\usepackage{nicefrac}       
\usepackage{microtype}      
\usepackage{xcolor}         
\usepackage{times}
\usepackage{enumitem}
\usepackage{notations/ISG_style}
\usepackage{notations/ISG_style_reduced}
\usepackage{amssymb}
\usepackage{bbm}

\usepackage{amsmath}
\usepackage{mathtools}
\usepackage{amsthm}

\usepackage{algorithm}
\usepackage{algorithmic}
\usepackage{subcaption}
\usepackage{caption}

\theoremstyle{plain}
\newtheorem{theorem}{Theorem}

\newtheorem{lemma}[theorem]{Lemma}
\newtheorem{corollary}[theorem]{Corollary}
\theoremstyle{definition}

\theoremstyle{remark}
\newtheorem{remark}[theorem]{Remark}

\newcommand{\algonameLinSEMUCB}{{\rm LinSEM-UCB}}
\newcommand{\algonameRWLinSEMUCB}{{\rm Robust-LCB}}

\title{\bf \LARGE Robust Causal Bandits for Linear Models}

\begin{document}
\maketitle
\allowdisplaybreaks
\begin{abstract}
Sequential design of experiments for optimizing a reward function in causal systems can be effectively modeled by the sequential design of interventions in causal bandits (CBs). In the existing literature on CBs, a critical assumption is that the causal models remain constant over time. However, this assumption does not necessarily hold in complex systems, which constantly undergo temporal model fluctuations. This paper addresses the \emph{robustness} of CBs to such model fluctuations. The focus is on causal systems with linear structural equation models (SEMs). The SEMs and the \emph{time-varying} pre- and post-interventional statistical models are all \emph{unknown}. Cumulative regret is adopted as the design criteria, based on which the objective is to design a sequence of interventions that incur the smallest cumulative regret with respect to an oracle aware of the entire causal model and its fluctuations. First, it is established that the existing approaches fail to maintain regret sub-linearity with even a few instances of model deviation. Specifically, when the number of instances with model deviation is as few as $T^\frac{1}{2L}$, where $T$ is the time horizon and $L$ is the length of longest causal path in the graph, the existing algorithms will have linear regret in $T$. For instance, when $T=10^5$ and $L=3$, model deviations in $6$ out of $10^5$ instances result in a linear regret. Next, a robust CB algorithm is designed, and its regret is analyzed, where upper and information-theoretic lower bounds on the regret are established. Specifically, in a graph with $N$ nodes and maximum degree $d$, under a general measure of model deviation $C$, the cumulative regret is upper bounded by $\tilde{\mathcal{O}}(d^{L-\frac{1}{2}}(\sqrt{NT} + NC))$ and lower bounded by  $\Omega(d^{\frac{L}{2}-2}\max\{\sqrt{T}\; , \; d^2C\})$. Comparing these bounds establishes that the proposed algorithm achieves nearly optimal $\tilde\mcO(\sqrt{T})$ regret when $C$ is $o(\sqrt{T})$ and maintains sub-linear regret for a broader range of $C$.
\end{abstract}

\section{Motivation and Overview}
Causal bandits provide a rich framework to formalize and analyze the sequential experimental design in causal networks. Such design problems appear in applications that involve a network of interacting components that can causally influence one another. Examples include design of experiments in robotics~\cite{ahmed2021causalworld}, gene expression networks~\cite{badsha2019learning}, drug discovery~\cite{liu2020reinforcement}, and recommendation systems~\cite{zhao2022mitigating}. In causal systems, 
\emph{interventions} are experimental mechanisms that facilitate uncovering the cause-effect relationships in causal networks and distinguishing them from the conventional association measures \cite{pearl2009causality}. Sequential design of interventions has the key advantage of bringing data-adaptivity in designing the interventions, resulting in an overall reduced experiment cost and a faster process for forming inferential decisions. For instance, reprogramming the cell via gene perturbation experiments needs a careful design of sequential interventions~\cite{zhang2021matching} such that the outcome of one experiment guides the design of the subsequent ones. Causal bandits (CBs) provide a theoretically principled way of sequentially designing interventions to identify one that maximizes a utility for the causal network. Specifically, the canonical model for this utility is a function of the observations obtained from the DAG. The utility is chosen as the average value of a leaf node, which we denote by the \textit{reward node}. This is a model that is widely used in the causal bandit literature. 
Accordingly, each intervention mechanism is modeled by an arm, the value of the reward node under an intervention is called the reward, and the sequential selection of the interventions is abstracted as arm selection decisions.

The extent of information available about the causal model critically influences the design of CB algorithms. Broadly, there are two central pieces of information: the causal structure (topology) and the data's pre- and post-intervention statistical models. Depending on the availability of each of these two pieces, there are four possible CB settings. Designing CB algorithms and their improvement over standard bandit algorithms was first demonstrated in~\cite{lu2021causal, lattimore2016causal} for the settings where both the structure and interventional distributions are fully specified. Subsequently, the studies transitioned to more realistic settings and explored how to accommodate unknown structure and distribution information~\cite{lu2021causal, de2022causal,feng2023combinatorialwioutgraph}, (approximately) known distributions but unknown structure~\cite{bilodeau2022adaptively, konobeev2023causal}, and known structure and unknown distributions~\cite{maiti2022causal,feng2023combinatorial,varici2022causal}. 

\vspace{.1 in}
\noindent\textbf{Motivation.} We investigate CBs from a new perspective. The existing studies all focus on having a \emph{fixed causal model over time}, which applies to both directions in which the models are known or unknown. In reality, however, large complex causal systems undergo model fluctuations caused by a wide range of reasons such as non-stationarity in the system or heterogeneous data~\cite{huang2017behind, zhang2017causalA}, 
measurement errors~\cite{zhang2017causalB}, selection bias~\cite{zhang2016identifiability}, and missing data~\cite{tu2019causal}. Temporal model fluctuations can change the causal structure or statistical models. For instance, in drug discovery, there are multiple observable variables or representation nodes~\cite{li2017learning}, and the model fluctuations due to measurement errors can occur in both the observable variables or their weights to the representation nodes. However, the algorithm for time-invariant settings can be highly susceptible to model fluctuations. For instance, the CB algorithm in \cite{varici2022causal} enjoys a nearly optimal growth in the horizon $T$, i.e., $\mcO(\sqrt{T})$. Nevertheless, it will lose the optimal rate with even minuscule instances of model deviation. More specifically, as we will show in Section~\ref{sec:comparison}, the regret becomes linear in $T$ if the system experiences model deviations in $T^{\frac{1}{2L}}$ instances, where $L$ is the longest causal path in the graph. For even small values of $L$, $T^{\frac{1}{2L}}$ will be an extremely small fraction of the instances. For instance, when $T=10^5$ and $L=3$, model deviations in $6$ out of $10^5$ instances result in a linear regret.

\vspace{.1 in}
\noindent\textbf{Objectives.} We pursue four objectives. (1) Under relevant measures of model fluctuations, we design a robust CB algorithm to model deviations over time. (2) We characterize almost matching upper and lower bounds on the regret as a function of model deviation level, time horizon, graph parameters, and the size of the cardinality of the intervention space. (3) We analytically assess the robustness of the relevant existing algorithms, establishing their lack of robustness to model fluctuations. (4) We consider a general intervention setting in which each subset of nodes in the graph is intervenable, and each intervention induces a distinct reward distribution, resulting in an intervention space that grows exponentially with the size of the graph. As the final objective, we show that our algorithm circumvents the exponential growth of the achievable regret with the cardinality of the intervention space and breaks it down to linear growth. Our focus is on the causal graphs that are specified by a linear structural equation model (SEM). We assume that the structure (topology) is fixed but the statistical models undergo temporal fluctuations.

\vspace{.1 in}
\noindent\textbf{Contribution \& key observations.} 
We design a novel weighting methodology for linear regression that takes advantage of the \emph{weighted} exploration bounce. This approach enables us to accurately accommodate the impact of model deviations in our regret analysis. Based on that, we characterize novel time-uniform confidence ellipsoid models for robust linear regression, which may be of broader interest in robust linear bandits. Furthermore, we propose a robust CB algorithm and analyze the compounding effect under the novel confidence ellipsoids, which offers insights into the behavior of our regret bound. When considering a \emph{known} budget of $C$ that captures the level of model deviations over time, the achievable regret is $\tilde{\mathcal{O}}(d^{L-\frac{1}{2}}(\sqrt{NT} + NC))$, where $N$ and $d$ are the number of nodes and maximum degree in the graph, respectively, and $L$ is the length of the largest causal path. Compared to the established lower bound $\Omega(d^{\frac{L}{2}-2}\max\{\sqrt{T}\; , \; d^2C\})$, we observe that both bounds scale polynomially in $d$ and exponentially in $L$. Furthermore, both bounds scale with $\sqrt{T}$ when the model faces $C=o(\sqrt{T})$ and linear in $C$ when $C=\Omega(\sqrt{T})$. This indicates that our algorithm enjoys \emph{nearly} optimal regret when $C=o(\sqrt{T})$, and it maintains a \emph{sub-linear} regret when the aggregate model deviation is sub-linear, which is the best possible regret order that any algorithm can achieve. The cost incurred to maintain such robustness is that the regret grows linearly with the deviation amount.

\vspace{.1 in}
\noindent\textbf{Related literature. } The earlier studies on CB algorithms assume that both the graph structure and the interventional distributions are known (fully or partially)~\cite{lu2020regret,nair2021budgeted}. More recent studies dispense with one or both of the assumptions. Despite their discrepancies in model and objectives, the common theme in all these studies is that they assumed a fixed causal model. Among the related work that does not make either assumption, ~\cite{de2022causal} incorporates causal learning algorithms to CBs but does not improve upon regret of non-causal bandit algorithms; ~\cite{lu2021causal} focuses on atomic interventions; ~\cite{bilodeau2022adaptively} achieves a regret bound that scales with the cardinality of the intervention space; and ~\cite{feng2023combinatorial} focuses on binary random variables. 

In a different direction, there exist studies that assume that the graph structure is known while the distributions are unknown. The relevant literature includes \cite{maiti2022causal,feng2023combinatorial}, which focuses on binary random variables. More recently, \cite{varici2022causal} focuses on linear systems and generalizes the results to the soft intervention settings, continuous random variables, and arbitrarily large intervention spaces. In parallel, \cite{sussex2022model} uses soft interventions and generalizes to non-linear models but limits to the Gaussian process SEMs in reproducing kernel Hilbert space (RKHS) and intervention space on controllable action variables. Finally, we note that even though we focus on linear SEMs, we observe that our reward is a non-linear function of the unknown parameters. Hence, our CB model fundamentally differs from linear bandits. This is the case even in the CB settings with a fixed model~\cite{varici2022causal}. Nevertheless, we briefly comment on the literature on linear bandits with model misspecification or contamination. These studies assume fixed (permanent) deviation, whereas, in our setting, the deviations can vary over time~\cite{ghosh2017misspecified,lattimore2020learning,foster2020adapting,krishnamurthy2021adapting}. Furthermore, in linear bandits with contamination, the contamination is imposed on the observed rewards~\cite{li2019stochastic,bogunovic2021stochastic,lee2021achieving,zhao2021linear,wei2022model,he2022nearly}, whereas we are focusing on model deviation. 

\vspace{.1 in}
\noindent\textbf{Notations.}
For $N\in \Z_{+}$, we define the set $[N]\triangleq\{1,\cdots,N\}$. The Euclidean norm of a vector $\bX\in\R^N$ is denoted by $\norm{\bX}$. For a subset $\mcS \subseteq [N]$, we define $\bX_{\mcS} \triangleq \bX\odot\mathbf{1}({\mcS})$, where $\odot$ denotes the Hadamard product and the vector $\mathbf{1}({\mcS})\in \{0,1\}^N$ has $1$s at the indices corresponding to $\mcS$. We denote the $i$-th column of matrix $\bA\in\R^{m\times n}$ by $[\bA]_i$, and the entry at $i$-th row and $j$-th column by $[\bA]_{i,j}$. The spectral norm of a matrix is denoted as $\norm{\bA}$. We further define the $\bA$-norm for positive semidefinite matrix $\bA$ as $\|\bX\|_{\bA}=\sqrt{\bX^{\top}\bA\bX}$.

\section{Causal Bandit Model}
\label{sec:intro}

\noindent\textbf{Causal graphical model.} Consider a directed acyclic graph (DAG) denoted by $\mathcal{G}(\mathcal{V},\mathcal{E})$, where $\mathcal{V}=[N]$ denotes the set of nodes, and $\mathcal{E}$ denotes the set of edges, where the ordered tuple $(i,j)\in\mathcal{E}$, indicates that there is a directed edge from $i$ to $j$. Each node $i\in[N]$ is associated with a random variable $X_i$. Accordingly, we define the vector $\mathbf{X}\triangleq [X_1,\cdots,X_N]^\top$. We consider a {\em linear} SEM, according to  which
\begin{equation}
\label{eq:linear model}
    \bX=\bB^{\top} \bX+\bepsilon \ ,
\end{equation}
where $\bB\in\R^{N\times N}$ is a strictly upper triangular edge weight matrix, and $\bepsilon\triangleq (\epsilon_1,\cdots,\epsilon_N)^\top$ denotes the exogenous noise variables, with a known mean $\bnu \triangleq \E[\bepsilon]$. The noise vector $\bepsilon$ is $1$-sub-Gaussian, and its Euclidean norm is upper bounded by $\norm{\bepsilon}\leq m_{\bepsilon}$. The graph's structure is assumed to be \emph{known}, while the weight matrix $\bB$ associated with the graph is \emph{unknown}. For any node $i\in[N]$, we denote the set of parents of~$i$ by ${\rm pa}(i)$. We denote the maximum in-degree of the graph by $d \triangleq \max_i\{|{\rm pa}(i)|\}$ and the length of the longest directed path in the graph by $L$.

\vspace{.1 in}
\noindent\textbf{Intervention model.} We consider \emph{soft interventions} on the graph nodes. A soft intervention on node $i\in[N]$ alters the conditional distribution of $X_i$ given its parents $\bX_{\rm pa}(i)$, i.e., $\P(X_i|\bX_{\rm pa}(i))$. An intervention can be applied to a subset of nodes simultaneously. If node $i\in\mcV$ is intervened, the impact of the intervention is a change in the weights of the edges incident on node $i$. These weights are embedded in $[\bB]_i$, i.e., the $i$-th column of $\bB$. We denote the post-intervention weight values by $[\bB^*]_i\neq [\bB]_i$. Accordingly, corresponding to the interventional weights, we define the interventional weight matrix $\bB^*$, which is composed of the columns $\{[\bB^*]_i:i\in[N]\}$. Note that soft interventions subsume commonly used stochastic \emph{hard} interventions in which a hard intervention on node $i$ sets $[\bB^*]_i=\boldsymbol{0}$.

Since we allow any arbitrary combination of nodes to be selected for concurrent intervention, there exist~$2^N$ interventional actions to choose from. We define $\mcA\;\triangleq\; 2^{\mcV}$ as the set of all possible interventions, i.e., all possible subsets of $[N]$. For any intervention $a\in\mcA$, we define the post-intervention weight matrix $\bB_a$ such that columns corresponding to the non-intervened nodes retain their observational values from $\bB$, and the columns corresponding to the intervened nodes change to their new interventional values from $\bB^*$. The columns of $\bB_a$ are specified by
\begin{equation}
\label{eq:Ba_construct}
    [\bB_a]_i\;\triangleq\;[\bB]_i\cdot\mathbbm{1}{\{i\notin a\}} + [\bB^*]_i\cdot\mathbbm{1}{\{i\in a\}}\ , 
\end{equation}
where $\mathbbm{1}$ denotes the indicator function. The interventions change the probability models of $\bX$. We define~$\P_a$ as the probability measure of $\bX$ under intervention $a\in\mcA$. For any given $\bB$ and $\bB^*$ we assume that $\norm{[\bB_a]_i} \leq m_B$. Without loss of generality, we assume $m_B = 1$. Due to the boundedness of noise $\bepsilon$ and column of $\bB_a$ matrices, there exists $m\in \R^+$ such that $\norm{\bX}\leq m$.

\vspace{.1 in}
\noindent\textbf{Causal bandit model.} Our objective is the sequential design of interventions. The set of possible interventions can be modeled as a multi-armed bandit setting with $2^N$ arms, one arm corresponding to each possible intervention. Following the canonical CB model~\cite{lattimore2016causal,lu2020regret}, we designate node $N$ (i.e., the node without a descendant) as the \emph{reward} node. Accordingly, $X_N$ specifies the reward value. We denote the expected reward collected under intervention $a\in\mcA$ by
\begin{equation}
\label{equ:mua}
    \mu_a \;\triangleq\; \E_a[X_N]\ ,
\end{equation}
where $\E_a$ denotes expectation under $\P_a$. We denote the intervention that yields the highest average reward by $a^* \triangleq \argmax_{a\in\mcA} \mu_a$; denote the sequence of interventions by $\{a(t)\in\mcA: t\in\N\}$; and denote the data generated at time $t$ and under intervention $a(t)$ by $\bX(t) = [X_1(t),\cdots,X_N(t)]^\top$. The learner's goal is to minimize the average cumulative regret over the time horizon $T$  with respect to the reward accumulated by an oracle aware of the systems model, interventional distributions, and model fluctuations. We define the expected accumulated regret as
\begin{equation}
    \E \left[R(T)\right]\;\triangleq\; T\mu_{a^*} - \sum\limits_{t=1}^T \E[X_N(t)]\ .
    \label{equ:regret}
\end{equation}
\section{Temporal Model Fluctuations}
Due to the size and complexity of the graphical models that represent complex systems, assuming that the observational and interventional models $\bB$ and $\bB^*$ remain unchanged over time is a strong assumption. These models can undergo temporal variations due to various reasons, such as model misspecifications, stochastic behavior of the system, and adversarial influences. To account for such variations, we refer to $\bB_{a(t)}$ as the nominal model of the graph at time $t$ and denote the actual time-varying \emph{unknown} model by $\bD_{a(t)}$. Accordingly, we define the deviation of the actual model from the nominal model by
\begin{equation}
 \Delta_{a(t)}  \;\triangleq\; \bD_{a(t)} - \mathbf{B}_{a(t)} \ . 
\end{equation}
To quantify the impact of model deviations on the regret $R(T)$, we specify two measures that capture the extent of deviations. The first measure captures the maximum number of times each node deviates from the nominal model. The second measure provides a budget for the maximum deviation in the linear model that model deviations can inflict over time. Clearly, if the model of node $i$ undergoes deviation at time $t$ under intervention $a$, we have $\big\|\left[\Delta_{a(t)}\right]_i\big\|\neq 0$.

\vspace{.1 in}
\noindent\textbf{Measure 1: Deviation Frequency (DF).} This measure accounts for how frequently each node's model can deviate from its nominal model, and over a horizon $T$ it is defined as
\begin{equation}
C_{\rm DF}  \;\triangleq\; \max_{i\in[N]} \sum_{t=1}^T \max_{a(t)\in\mcA} \mathbbm{1}\left\{\big\|\left[\Delta_{a(t)}\right]_i\right\|\neq 0\big\}\ .
\end{equation}
This model is adopted from misspecified bandit literature~ \cite{lattimore2020learning}. To avoid unbounded deviations, we assume that the deviation inflicted on each node at any given time is bounded by a constant $m_c\in\R_+$, i.e., 
\begin{equation}
    \label{equ:maxdeviation}
    \max\limits_{i\in[N]} \max\limits_{t\in[T]} \max\limits_{a(t)\in\mcA}\big\|\left[\Delta_{a(t)}\right]_i\big\| \leq m_{\rm c}\ .
\end{equation}

\noindent\textbf{Measure 2: Aggregate Deviation (AD).} This measure quantifies the aggregate deviation over time. Specifically, we define the maximum aggregate deviation as
\begin{equation}
C_{\rm AD} \;\triangleq\; \max_{i\in[N]} \sum_{t=1}^T\max_{a(t)\in\mathcal{A}}\left\|\left[\Delta(t)\right]_i\right\|\ .
\end{equation} 
This measure of deviation is also standard in stochastic bandits~\cite{he2022nearly}, where the deviation budget is defined as the maximum deviation in the reward that the adversary can inflict over time. We will observe that $C_{\rm DF}$ and $C_{\rm AD}$ impact the regret results similarly. Hence, to unify the results and present them in a way that applies to both measures, we use $C$ to represent the level of model deviation. For measure 1, we define $C$ as the product of a constant factor $m_{\rm c}$ and $C_{\rm DF}$, while for measure 2, we define $C$ as $C_{\rm AD}$. We assume that the model deviation budgets specified by $C$ are known to the learner, allowing the CB algorithm to adapt to the varying levels of model deviation.

\section{\algonameRWLinSEMUCB{}~Algorithm}
\label{sec:algo}

In this section, we present the details of our algorithm and provide the performance guarantee (regret analysis) in Section~\ref{sec:regret}. We also provide theoretical comparisons to the existing algorithms designed for CBs with fixed models, establishing their lack of robustness against model variations.

\vspace{.1 in}
\noindent\textbf{Algorithm overview.} Identifying the best intervention hinges on determining which of the distributions $\{\P_a:a\in\mcA\}$ maximizes the expected reward. Nevertheless, these $2^N$ distributions are unknown. Therefore, a direct approach entails estimating these probability distributions, the complexity of which grows exponentially with $N$. To circumvent this, we leverage the fact that specifying these distributions has redundancies since all depend on the observational and interventional matrices $\bB$ and $\bB^*$. These matrices can be fully specified by $2Nd$ scalars, where $d$ is the maximum degree of the causal graph. Hence, at its core, our proposed approach aims to estimate these two matrices. 

We design an algorithm that has two intertwined key objectives. One pertains to the \emph{robust estimation} of matrices $\bB$ and $\bB^*$ when the observations are generated by the non-nominal models. For this purpose, we design a weighted ordinary least squares (W-OLS) estimator. The structure of the estimator and the associated confidence ellipsoids for the estimates are designed to circumvent model deviations effectively. The second objective is designing a decision rule for the sequential selection of the interventions over time. This sequential selection, naturally, is modeled as a multi-armed bandit problem. Therefore, we design an upper confidence bound (UCB)--based algorithm for the sequential selection of the interventions over time. Next, we present the details of the {\bf Robust} {\bf L}inear {\bf C}ausal {\bf B}andit (\algonameRWLinSEMUCB). The steps involved in this algorithm are summarized in Algorithm~\ref{alg:weighted_ucb_algorithm}.

\begin{algorithm}[ht]
\caption{\algonameRWLinSEMUCB}
\label{alg:weighted_ucb_algorithm}
\begin{algorithmic}[1]
   \STATE {\bfseries Input:} Horizon $T$, causal graph $\mcG$, action set $\mcA$, mean noise vector $\bnu$, deviation budget  $C$
   \STATE {\bfseries Initialization}: Initialize 
   $\bB(0) = \bB^{*}(0) =\mathbf{0}_{N \times N}$ and $\bV_{i}(0)= \bV^{*}_{i}(0) = \bI_N, \ \forall i \in [N] $. 
   \FOR{$t = 1, 2,\ldots,T$}
   \STATE Compute ${\rm UCB}_a(t)$ according to \eqref{eq:ucb_definition} for $a\in \mcA$.
   \STATE Pull $a(t) = \argmax_{a \in \mcA} {\rm UCB}_a(t)$ and observe $\bX(t)\!=\!(X_1(t),\dots, X_N(t))^\top$. 
   \FOR{$i \in \{1,\dots,N\}$}
        \STATE Set $w_{i}(t)$ as \eqref{equ:weightsAC}, update $[\bB(t)]_{i}$ according to \eqref{eq:estimate_obs_weighted} and update $[\bB^{*}(t)]_{i}$ according to \eqref{eq:estimate_int_weighted}.
   \ENDFOR
   \ENDFOR
\end{algorithmic}
\end{algorithm}

\vspace{.1 in}
\noindent\textbf{Countering model deviations.} Our approach to circumventing model deviations is to identify and filter out the samples generated by the non-nominal models. We refer to these samples as \emph{outlier samples}. This facilitates forming estimates for $\bB$ and $\bB^*$ based on the samples generated by the nominal models. Since the model deviations may happen on multiple nodes simultaneously, the \algonameRWLinSEMUCB\ is designed to identify the nodes undergoing deviation over time and discard the outlier samples generated by these nodes. Such filtration is implemented via assigning time-varying and data-adaptive weights to different nodes such that the weight assigned to node $i\in[N]$ at time $t\in \N$ balance two factors: the probability of node $i\in[N]$ undergoing deviation at $t$ and the contribution of that sample to the estimator. These weights, subsequently, control how the samples from different nodes contribute to estimating $\bB$ and $\bB^*$.

\vspace{.1 in}
\noindent\textbf{Robust estimation.} We design the {\em weighted} ordinary least squares (OLS) estimators for $\bB$ and $\bB^*$, which at time $t\in\N$ are denoted by $\bB(t)$ and $\bB^*(t)$, respectively. To estimate the observational weights $[\bB]_i$, we use the samples from instances at which node $i$ is not intervened. Conversely, to estimate the interventional weights $[\bB^*]_i$, we use the samples from the instances at which node $i$ is intervened. By defining $\{w_i(t)\in\R_+:i\in[N]\}$ as the set of weights assigned to the nodes at time $t\in\N$, $i$-th columns of these estimates are specified as follows. 
\begin{align}
[\bB(t)]_{i} &  \;\triangleq\;  [\bV_{i}(t)]^{-1}   \sum_{s \in [t] : i \notin a(s)} w_i(s)\bX_{\Pa(i)}(s) (X_i(s)-\nu_i) \ , \label{eq:estimate_obs_weighted} \\
\hspace{-0.15in}[\bB^{*}(t)]_i & \; \triangleq\; [\bV^{*}_{i}(t)]^{-1}  \sum_{s \in [t] : i \in a(s)}  w_i(s)\bX_{\Pa(i)}(s) (X_i(s)-\nu_i) \ , \label{eq:estimate_int_weighted}
\end{align}
where we have defined the \emph{weighted Gram matrices} as
\begin{align}
\bV_{i}(t) & \;\triangleq \sum_{s \in [t] : i \notin a(s)} w_i(s) \bX_{\Pa(i)}(s) \bX_{\Pa(i)}^\top(s) + \bI_N \ , \label{eq:define_V_obs_weighted} \\
\bV^{*}_{i}(t) & \;\triangleq\sum_{s \in [t] : i \in a(s)} w_i(s) \bX_{\Pa(i)}(s) \bX_{\Pa(i)}^\top(s)+\bI_N \  . \label{eq:define_V_int_weighted}
\end{align}
Furthermore, we define the matrices associated with the squared weights as
\begin{align}
\widetilde{\bV}_{i}(t) & \;\triangleq \sum_{s \in [t] : i \notin a(s)} w_i^2(s) \bX_{\Pa(i)}(s) \bX_{\Pa(i)}^\top(s) + \bI_N \ , \label{eq:define_V_obs_weighted_tilde} \\
\widetilde{\bV}^{*}_{i}(t) & \;\triangleq\sum_{s \in [t] : i \in a(s)} w_i^2 (s)\bX_{\Pa(i)}(s) \bX_{\Pa(i)}^\top(s)+\bI_N \  . \label{eq:define_V_int_weighted_tilde}
\end{align}
Similarly to \eqref{eq:Ba_construct}, we denote the relevant and Gram matrices for node $i$ under intervention $a \in \mcA$ by
\begin{align}
     \widetilde{\bV}_{i,a}(t) &\;\triangleq \;\mathbbm{1}{\{i \in a\}} \widetilde{\bV}^{*}_{i}(t) +  \mathbbm{1}{\{i \notin a\}} \widetilde{\bV}_{i}(t) \ . \label{eq:V_ita_tilde}
\end{align}
\textbf{Confidence ellipsoids.} After performing estimation in each round, we construct the confidence ellipsoids for the OLS estimators $\{\mcC_{i}(t):i\in[N]\}$ for the observational weights and $\{\mcC^*_{i}(t):i\in[N]\}$ for the interventional weights 
\begin{align}
    \mcC_{i}(t) & \; \triangleq \; \left\{\theta \in \mcB_1  :  \LARGE\|\theta - [\bB(t-1)]_{i}\LARGE\|_{\bV_i(t-1) [\widetilde{\bV}_i(t-1)]^{^{-1}} \bV_i(t-1)}  \leq \beta_{t} \right\}  , \label{eq:conf_obs} \\
    \mcC^*_{i}(t)  & \;\triangleq\; \left\{ \theta \in \mcB_1: \LARGE\|\theta - [\bB^{*}(t-1)]_{i}\LARGE\|_{\bV^*_i(t-1) [\widetilde{\bV}^*_i(t-1)]^{^{-1}} \bV^*_i(t-1)} \leq \beta_{t} \right\} \label{eq:conf_int} ,
\end{align}
where $\mcB_1$ is the unit ball in $\R^N$ and $\{\beta_{t}\in\R_+, t\in\N\}$ is a sequence of non-decreasing confidence radii that control the size of the confidence ellipsoids, which we will specify. Accordingly, we define the relevant confidence ellipsoid for node $i$ under intervention $a\in\mcA$ as 
\begin{equation}
    \mcC_{i,a}(t) \;\triangleq\;  \mathbbm{1}{\{i \in a\}}\  \mcC^{*}_{i}(t) + \mathbbm{1}{\{i \notin a\}}\  \mcC_{i}(t) \ . \label{eq:conf_C_iat}
\end{equation}
\textbf{Weight designs.} Designing the weights $\{w_i(t):i\in [N] \}$ at time $
t$ is instrumental in effectively winnowing out the outlier samples. We select the weights that bring the confidence radius $\beta_t$ down to nearly constant
\begin{equation}
    w_i(t)\;\triangleq\; \min\bigg\{\frac{1}{C} \; , \; \frac{1}{C\norm{\bX_{\Pa(i)}(t)}_{[\widetilde{\bV}_{i,a(t)}(t)]^{^{-1}}}}\bigg\}\ ,\label{equ:weightsAC}
\end{equation}
where the weights are inversely proportional to the norm $\|\bX_{\Pa(i)}(t)\|_{[\widetilde{\bV}_{i,a(t)}(t)]^{^{-1}}}$ and deviation budget $C$ , and they are truncated at $1/C$, which ensures that the weights are not arbitrarily large. We refer the term $\|\bX_{\Pa(i)}(t)\|_{[\widetilde{\bV}_{i,a(t)}(t)]^{^{-1}}}$ as {\em weighted exploration bonus}.
A higher exploration bonus means lower confidence in the sample. Setting the weights as the inverse of the exploration bonus avoids potentially significant regret caused by both the stochastic noise and model deviations. We scale the weights proportional to $1/C$ to use smaller weights when the model deviation level is higher.

\vspace{.1 in}
\noindent\textbf{Intervention selection.} We adopt a UCB-based rule for sequentially selecting the interventions. 
Specifically, at each time $t$, our algorithm selects the intervention that maximizes a UCB, defined as the maximum value of expected reward when the edge weights are in the confidence ellipsoids $\{\mcC_{i,a}(t),i\in[N]\}$, under that intervention. 
Recall the expected reward $\mu_a$ for $a\in\mcA$ defined in \eqref{equ:mua} is a function of the edge weights $\bB_a$, which can be decomposed according to the following lemma. 
\begin{lemma}{~\cite[Lemma 1]{varici2022causal}}\label{lem:expectedvalue}\label{lm:expected_reward}
 Consider a linear SEM $\mcG(\mcV,\mcE)$ with intervention $a\in2^\mcV$, whose weight matrix is denoted by $\bB_a\in\R^{N\times N}$. Furthermore, consider the function $f(\bA)\triangleq \sum_{\ell = 0}^L [\bA^\ell]_N$ for $\bA\in\R^{N\times N}$, where we denote $\bA^{\ell}$ as the $\ell$-th power of matrix $\bA$. We have
 \begin{equation} 
     \mu_a\;=\; \inner{f(\bB_a)}{\bnu}\ ,
\end{equation}
 where $\bnu= (\nu_1,\cdots,\nu_N)$, and $\nu_i= \E[\epsilon_i]$ denotes the mean of the noise vector corresponding to node $i\in [N]$.
\end{lemma}
Thus, for any intervention $a\in\mcA$, the UCB is naturally defined as
\begin{equation}
\label{eq:ucb_definition}
    {\rm UCB}_a(t) \;\triangleq \max_{\{\forall i\in[N]: [\Theta]_i\in \mcC_{i,a}(t)\}}\inner{f(\Theta)}{\bnu }\ .
\end{equation}
Based on the UCB in~\eqref{eq:ucb_definition}, at time $t$, our algorithm selects the intervention that maximizes the UCB, 
\begin{equation}
    a(t)\;=\;\argmax_{a \in \mcA} {\rm UCB}_a(t)\ .
\end{equation}

\section{Regret Analysis}
\label{sec:regret}
In this section, we present the performance guarantees for the proposed \algonameRWLinSEMUCB~ algorithm. We first provide the upper bound on the average cumulative regret in Section~\ref{sec:upperbound}. We also establish a minimax lower bound in Section~\ref{sec:lowerbound} that shows the tightness of our upper bound. By comparing our regret with that of \algonameLinSEMUCB~in Section~\ref{sec:comparison}, we evaluate the robustness of our algorithm. 

\subsection{Regret Upper Bound}
\label{sec:upperbound}
In order to derive the upper bound, we begin by providing a 
concentration bound for the W-OLS estimator. Notably, we investigate a vector norm that differs from existing work in robust bandits. This norm was first investigated in \cite{russac2019weighted} under the non-stationary setting, and our investigation builds on this to provide novel insights into the robust behavior of the W-OLS.

\begin{lemma}[Estimator concentration]
\label{lem:beta_Tindeviation}
Under a deviation budget $C$, with a probability at least $1-2\delta$, for any node $i\in[N]$ and $t\geq 0$, we have
\begin{align}
        \|[\bB(t)]_i-\bB_i\|_{\bV_i(t) [\widetilde{\bV}_i(t)]^{-1} \bV_i(t)}  \;&\leq\; \beta_t(\delta)\ ,\\
        \mbox{and}\quad\|[\bB^*(t)]_i-\bB^*_i\|_{\bV_i^*(t) [\widetilde{\bV}_i^*(t)]^{-1} \bV_i^*(t)} \;&\leq\; \beta_t(\delta)\ ,
    \end{align}
where we have defined $\beta_t(\delta)\;\triangleq\;\sqrt{2\log\left(1/\delta\right)+d\log\left(1+m^2t/dC^2\right)} + 1 + m$ .
\end{lemma}
\begin{proof}
We will provide the proof corresponding to the observational weights $[\bB(t)]_i$, while the proof for the interventional weights $[\bB^*(t)]_i$ follows similarly.
For any node $i\in[N]$, we decompose the error in estimation $\|[\bB(t)]_i-\bB_i\|_{\bV_i(t) [\widetilde{\bV}_i(t)]^{-1} \bV_i(t)}$ for $t\geq 0$ as follows.
\begin{align}
   \hspace{-0.5 cm}  &\|[\bB(t)]_i-[\bB]_i\|_{\bV_i(t) [\widetilde{\bV}_i(t)]^{-1} \bV_i(t)}\\
     &\hspace{-0.5 cm}=\biggl\|[\bV_{i}(t)]^{-1}  \!\!\!\!\! \sum_{s\in[t],i \notin a(t)}\! \!\!\! w_i(s) \bX_{\Pa(i)}(s) \left[ \bX^{\top}_{\Pa(i)}(s) [\bD(t)]_i+\epsilon_{i}(s)-\nu_i \right]-[\bB]_i\biggl\|_{\bV_i(t) [\widetilde{\bV}_i(t)]^{-1} \bV_i(t)}\\
     &\hspace{-0.5 cm}\leq\underbrace{\biggl\| [\widehat{\bB}(t)]_i-[\bB]_i\biggl\|_{\bV_i(t) [\widetilde{\bV}_i(t)]^{-1} \bV_i(t)}}_{I_1\text{: Stochastic and regularization error}} +\underbrace{\left\|\sum_{s\in[t],i\notin a(t)}\!\!  \bX_{\Pa(i)}(s) \bX^{\top}_{\Pa(i)}(s)[\Delta(s)]_i\right\|_{[\widetilde{\bV}_{i}(t)]^{-1}}}_{I_2\text{: Fluctuation error}}\label{equ:errorDecomposition_RW} \ .
\end{align}
where $\widehat{\bB}(t)$ refers to the auxiliary estimators which correspond to the ridge regression estimator but with the removal of deviation's impact on the output, i.e.,
\begin{equation}
    [\widehat{\bB}(t)]_i = [\bV_{i}(t)]^{-1}  \!\!\!\! \sum_{s\in[t],i \notin a(t)} \!\!\!\! w_i(s) \bX_{\Pa(i)}(s) \left[ \bX^{\top}_{\Pa(i)}(s)[\bB]_i+\epsilon_{i}(s)-\nu_i \right]\ .
\end{equation}
The stochastic and regularization errors can be bounded by the following lemma.

\begin{lemma}
\label{lem:beta_T_timeinvariant}
For all node $i\in[N]$, with probability at least $1-\delta$, $\forall t\geq 0$, we have
\begin{align}
   I_1&=\left\|[\widehat{\bB}(t)]_i-[\bB]_i\right\|_{\bV_i(t) [\widetilde{\bV}_i(t)]^{-1} \bV_i(t)} \\
   &\leq  1+ \sqrt{2 \log (1 / \delta)+d \log \left(1+\frac{m^2 \sum_{s\in[t-1],i \notin a(t)} w_i^2(s)}{d}\right)} \ .
\end{align}
\end{lemma}
\begin{proof}
    Note that the weights $w_i(t)$ are predictable, i.e., $\mcF_i(t-1)$ measurable, if the $\sigma$-algebra is defined as $\mcF_i(t)=\sigma(\bX_{\Pa(i)}(1), \epsilon_i(1), \bX_{\Pa(i)}(2), \epsilon_i(2),\cdots,\bX_{\Pa(i)}(t) , \epsilon_i(t), \bX_{\Pa(i)}(t+1))$ similarly to the one used in~\cite{abbasi2011improved}. This modification of the filtration allows weights to depend on the current value of $\bX_{\Pa(i)}(t)$. Then the lemma results from \cite[Theorem 1]{russac2019weighted} with $\mu_t=\lambda_t=1$. 
\end{proof}
Since $w_i(s)\leq \frac{1}{C}$ for all $s\in[t]$ and $i\in[N]$, we can further upper bound the stochastic and regularization error as follows.
\begin{align}
      I_1  &\leq 1+ \sqrt{2 \log (1 / \delta)+d \log \left(1+\frac{m^2 t}{d C^2}\right)} \ .
\end{align}
Now we need to bound the fluctuation error $I_2$, which is bounded as follows.
\begin{align}
    I_2&=\left\|[\widetilde{\bV}_{i}(t)]^{-1/2} \sum_{s\in[t],i\notin a(t)} w_{i}(s) \bX_{\Pa(i)}(s) \bX^{\top}_{\Pa(i)}(s)[\Delta_\bB(s)]_i\right\|\\
    \label{eq: E1}
    &\leq \sum_{s\in[t],i\notin a(t)} w_{i}(s) \left\|[\widetilde{\bV}_{i}(t)]^{-1/2}  \bX_{\Pa(i)}(s) \bX^{\top}_{\Pa(i)}(s)[\Delta_\bB(s)]_i\right\|\\
    \label{eq: E2}
    &= \sum_{s\in[t],i\notin a(t)}  w_{i}(s) \left\|[\widetilde{\bV}_{i}(t)]^{-1/2}  \bX_{\Pa(i)}(s)\right\| \left| \bX^{\top}_{\Pa(i)}(s)[\Delta_\bB(s)]_i\right|\\
    \label{eq: E3}
    &\leq m  \sum_{s\in[t],i\notin a(t)} w_{i}(s)\norm{\Delta_\bB(s)]_i}  \left\| \bX_{\Pa(i)}(s)\right\|_{[\widetilde{\bV}_{i}(t)]^{-1}}\\
    \label{eq: E4}
    &\leq m \sum_{s\in[t],i\notin a(t)} w_{i}(s) \norm{\Delta_\bB(s)]_i}  \left\| \bX_{\Pa(i)}(s)\right\|_{[\widetilde{\bV}_{i}(s)]^{-1}}\\
    &\leq m \ , \label{UCB:boundI2}
\end{align}
where~\eqref{eq: E1} follows from the triangle inequality, \eqref{eq: E3} follows from the fact that $\norm{\bX_{\Pa(i)}(s)}\leq m$, \eqref{eq: E4} holds since we have $\norm{x}_{[\widetilde{\bV}_{i}(t)]^{-1}}\leq \norm{x}_{[\widetilde{\bV}_{i}(s)]^{-1}}$ for any $s\in[t]$ and $x\in\R^{N}$, and~(\ref{UCB:boundI2}) is obtained using the definition of the sequence of weights $\{w_i(s) : s\in[t]\}$. Finally, substituting the results of Lemma~\ref{lem:beta_T_timeinvariant} and \eqref{UCB:boundI2}, with probability at least $1-\delta$, for all $t\geq 0$, we have
\begin{equation}
\label{equ:obs_ellipoisd}
    \|[\bB(t)]_i-\bB_i\|_{\bV_i(t) [\widetilde{\bV}_i(t)]^{-1} \bV_i(t)}\leq 1 + m+ \sqrt{2\log\left(1/\delta\right)+d\log\left(1+m^2t/dC^2\right)} \ .
\end{equation}
Similarly, for the estimators for interventional weights, with probability at least $1-\delta$, for all $t\geq 0$, we have
\begin{equation}
\label{equ:int_ellipoisd}
    \|[\bB^*(t)]_i-\bB^*_i\|_{\bV^*_i(t) [\widetilde{\bV}^*_i(t)]^{-1} \bV^*_i(t)}\leq 1 + m+ \sqrt{2\log\left(1/\delta\right) +d\log\left(1+m^2t/dC^2\right)} \ .
\end{equation}
Combining the results in \eqref{equ:obs_ellipoisd} and \eqref{equ:int_ellipoisd} we complete the proof.
\end{proof}
The previous lemma offers high probability error bounds for estimators. Due to the causal structure, these errors accumulate and propagate along the causal path, leading to the reward node $N$. Consequently, we analyze the compounding impacts of estimation errors and model deviations. This analysis involves examining the eigenvalues of the weighted Gram matrices $\bV_{i,a(t)}(t)$ and $\widetilde{\bV}_{i,a(t)}(t)$. We introduce the subsequent lemma to show a bound on the accumulated estimation errors on the reward node with proof provided in Section~\ref{proof:lm:bound_l_paths}. 
\begin{lemma} \label{lm:bound_l_paths}
For any given intervention $a\in\mcA$ matrices $\bA \in\R^{N\times N}$ and $\bM_i\in\R^{N\times N}$ for all $i \in [N]$ , define
\begin{equation}
\label{eq:def_delta_lemma}
    \Delta_{\bA} \;\triangleq\; \bA - \bB_a\ , \quad \mbox{and} \quad \Delta_{\bA}^{(\ell)} \;\triangleq\; \bA^{\ell} - \bB_a^{\ell}\ .
\end{equation}
If $\bA$ shares the same support with $\bB_a$, $\bM_i \succeq  \bI$ and $[\bM_i]_j=[\bM_i^{\top}]_j=\be_i$ if $[\bB]_{j,i}=0$, and if the following bound holds
\begin{equation}
\norm{[\Delta_{\bA}]_{i}}_{\bM_{i}}\leq~\beta \ ,\end{equation}
then for all $\ell \in [L]$, we have 
\begin{equation}
\norm{\big[\Delta_{\bA}^{(\ell)}\big]_{N}} < d^{\frac{\ell-1}{2}} (\beta+1)^{\ell} \max_{i \in [N]} \lminn{\bM_{i}}{-\frac{1}{2}}  \ .
\end{equation}
\end{lemma}
Next, building on the estimation error bounds established in Lemma~\ref{lem:beta_Tindeviation} and the compounding error bounds established in Lemma~\ref{lm:bound_l_paths}, we derive a unified regret bound that applies to both measures of model deviation.  It is noteworthy that the analysis is distinctly different from that in the time-invariant setting since we are facing model fluctuations, for which we have designed novel weights for the W-OLS estimator.

\begin{theorem}[Regret upper bound]
\label{thm:measure2}
 Under a deviation budget $C$, by setting $\delta=\frac{1}{2NT}$ and $\beta_t(\delta)$ according to Lemma~\ref{lem:beta_Tindeviation}, the average cumulative regret of \algonameRWLinSEMUCB\  is upper bounded by
\begin{equation}
\label{equ:upperbound}
    \E \left[R(T)\right]\leq 2m+\tilde{\mathcal{O}}\left(d^{L-\frac{1}{2}}(\sqrt{NT} + NC)\right) \ .
\end{equation}
\end{theorem}
\noindent\emph{Proof sketch.} Characterizing the regret bound involves decomposing the regret into two parts, depending on whether the concentration inequality in Lemma~\ref{lem:beta_Tindeviation} holds. If the concentration inequality does not hold, the regret is upper bound by a constant term. Otherwise, the estimation errors are upper bounded in Lemma~\ref{lem:beta_Tindeviation} with their compounding effects on the reward node $N$ upper bounded by Lemma~\ref{lm:bound_l_paths}.  Next, the behavior of the eigenvalues of weighted Gram matrices  $\bV_{i,a(t)}(t)$ and $\widetilde{\bV}_{i,a(t)}$ is investigated to reach the final result. See Appendix~\ref{proof:thm:measure2} for the detailed proof.

The regret bound derived in Theorem~\ref{thm:measure2} can be decomposed into two parts. The first term 
recovers the order of the optimal rate achieved in the time-invariant setting. The second term captures the impact of model deviation on the regret bound, that is, the cost of handling unknown model fluctuations. Next, we present a lower bound that confirms the tightness and optimality of our upper bound.  

\begin{remark}
    Robust-LCB works when replacing the deviation budget $C$ with an upper bound $\bar C$. All the  analyses and performance guarantees remains valid when $C$ is substituted with its upper bound $\bar C$.
\end{remark}

\subsection{Regret Lower Bound}
\label{sec:lowerbound}
For our analysis of the lower bound, we first show the tightness of the second term of the upper bound that captures the model deviation level $C$. Building on this insight, we then combine this with the existing lower bound in the time-invariant setting to show the tightness of our regret bound.
\begin{theorem}[Regret lower bound]
    \label{thm:lowerbound}
    For any degree $d$ and graph length $L$, for any algorithm with knowledge of $C$, there exists a bandit instance such that the expected regret is at least
    \begin{equation}
         \E\left[R(T)\right]\geq \Omega(d^{\frac{L}{2}}C)\ .
    \end{equation}
\end{theorem}
\begin{proof}
   We construct two instances of causal bandits and demonstrate that under specific deviations, no algorithm can distinguish between them and the initial stage. We consider two linear SEM causal bandit instances sharing the same hierarchical graph $\mcG$ as shown in Figure~\ref{fig:block-hierarchical}. 
    \begin{figure}[ht]
    \centering
    \includegraphics[height=5 cm]{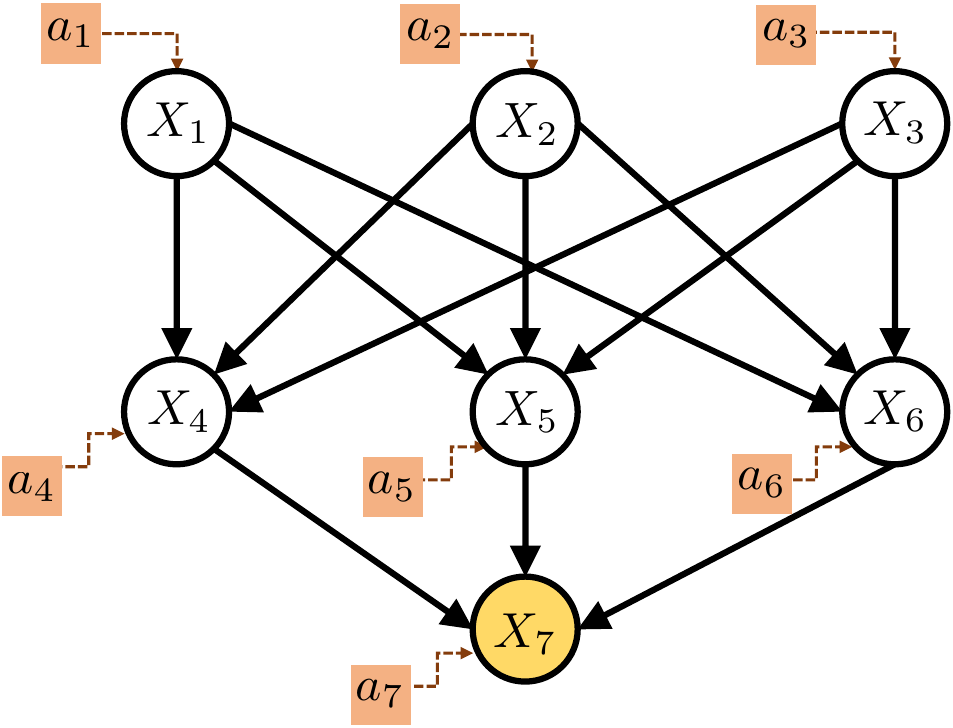}
    \caption{Example of the hierarchical graph with graph degree $d=3$ and length $L=2$.}
    \label{fig:block-hierarchical} 
    \end{figure}
Let us examine the parameterization of the two bandit instances, referred to as $\mcI_1=\{\bB, \bB^*, \epsilon\}$ and $\mcI_2=\{\bar{\bB}, \bar{\bB}^*, \epsilon\}$. For the existing edges in graph $(i,j)\in\mcE$ for $i<j$ and $i,j\in[N]$, we define
\begin{itemize}
    \item If $j<N$:
    \begin{align}
        [\bB]_{i,j} & =[\bar{\bB}]_{i,j} = [\bB^*]_{i,j} = [\bar{\bB}^*]_{i,j}=\sqrt{\frac{1}{d}} \ .
    \end{align}
    \item If $j=N$:
     \begin{align}
        [\bB^*]_{i,j} & =-[\bar{\bB}^*]_{i,j}=\sqrt{\frac{1}{d}} \ ,\quad \mbox{and}  \quad          [\bB]_{i,j}  =[\bar{\bB}]_{i,j}=0 \ .
    \end{align}
\end{itemize}
For the noises, we define
    \begin{align} 
        \epsilon_i \sim \left\{
        \begin{array}{cc}
          \mcN(1,1)   &  \text{for} \quad i\in\{1,\cdots, d\}\\
           \mcN(0,1)  &  \text{for} \quad i\notin\{1,\cdots, d\}\\
        \end{array}
        \right.  \ .
    \end{align}    
Thus, the only difference between the two bandit instances lies in the intervention weights assigned to the reward node. In the first bandit instance, the optimal action is when the reward node is \emph{intervened}. In contrast, in the second bandit instance, the best action is associated with the reward node being \emph{not intervened}. The regret incurred from choosing the sub-optimal action on the reward node is $d^{L/2}$. Next, consider the scenario where deviations only occur on the reward node during the initial $C$ rounds. Furthermore, consider the scenario in which, at each time $t\in\{1,\cdots,C\}$, the weights are set to $0$, and the learner only observes a random noise as the reward. In this case, the learner has no information about the $[\bB]_N$ and can only make random guesses. After $C $ rounds, either $\E N_{i}(C )$ or $\E N^*_{i}(C )$ is no less then $\frac{C }{2}$. Consequently, there must exist a bandit instance at which the algorithm plays the sub-optimal arm at least $\frac{C }{2}$ times. We conclude that it must incur $\frac{1}{2}d^{\frac{L}{2}}C$ regret with probability at least $1/2$. Ignoring the constant, by using Markov's inequality, we have
\begin{equation}
    \E\left[ R(T)\right]\geq \frac{1}{2}d^{\frac{L}{2}}C \times \P\left(R(T)\geq \frac{1}{2}d^{\frac{L}{2}}C\right) = \Omega(d^{\frac{L}{2}}C) \ .
\end{equation}
\end{proof}

\begin{theorem}\cite[Theorem~6]{varici2022causal}   \label{thm:lowerbound_invariant}
    For any degree $d$ and graph length $L$ and any algorithm, there exists a causal bandit instance such that the expected regret is at least
    \begin{equation}
        \E \left[R(T)\right]\geq \Omega(d^{\frac{L}{2}-2} \sqrt{T}) \ .
    \end{equation}
\end{theorem}

Combining the results in Theorem~\ref{thm:lowerbound} and Theorem~\ref{thm:lowerbound_invariant}, we can conclude a minimax lower bound in the following corollary.
\begin{corollary}
    For any degree $d$ and graph length $L$ and any algorithm, there exists a causal bandit instance such that the expected regret is at least
\begin{align}
    \E\left[R(T)\right]\geq \Omega(d^{\frac{L}{2}-2}\max\{\sqrt{T}\; , \; d^2C\})\ .
\end{align}
\end{corollary}
This corollary shows both the lower bounds depends on $\sqrt{T}$, and that the regret bound of our proposed algorithm is tight in terms of $T$ and $C$. Besides these, however, the achievable and lower bounds have a gap due to mismatching dependence on the number of nodes $N$ and the exact order of the exponential scaling of dimension $d$ with graph length $L$.  The dependence on $N$ arises in the techniques used in Theorem~\ref{thm:measure2} and we conjecture that the dependency of the upper bound on $N$ can be diminishing as $T$ grows. We provide some insights into tightening dependence on $N$. Let $L_i$ denote the length of the longest causal path that ends at node $i \in [N]$. If we can first bound the cumulative estimation error for the expected value of node $i$ with $L_i=1$, then we can use induction to bound that for increasing $L_i$ and derive an regret upper bound independent of $N$. The mismatch in the exact order $d$ with $L$ exists in all the relevant literature, even in simpler settings. For instance, consider linear bandits with time-invariant models with dimension $d$, which can be considered a special case of our linear causal bandit model by setting $L=1$ and no model variations. For the widely used optimism in the face of uncertainty linear bandit algorithm (OFUL) in \cite{abbasi2011improved}, the lower and upper bounds behave according to $\tilde \mcO(\sqrt{dT})$ and $\tilde \mcO (d\sqrt{T})$, respectively. This gap in terms of $\sqrt{d}$ and $d$ matches exactly our gap.

\subsection{Comparison with the Time-invariant Setting}
\label{sec:comparison}
To highlight the robustness, we compare the analytical results with those of the time-invariant setting~\cite{varici2022causal}. While our results in terms of the deviation budgets are general, for illustration purposes, we consider deviation budgets that scale sub-linearly with respect to $T$ by setting $C=T^\alpha$ for $\alpha\in(0,1)$.

\vspace{.1 in}
\noindent\textbf{Confidence ellipsoids.} In the time-invariant setting, the confidence ellipsoid radius is set to 
\begin{equation}
\label{equ:beta_Lin}
    \beta'_T\;\triangleq\; 1+ \sqrt{2\log\left(2NT\right)+d\log\left(1+mT^2/d\right)}\ .
\end{equation}
Compared with this, our choice of confidence radius is similar to the time-invariant setting. The main contrast is the inclusion of an additional term $m$, which accounts for the cost of robustness. The difference in logarithmic terms arises from using different weights and norms in our algorithm.

\vspace{.1 in}
\noindent\textbf{Regret bounds.} In the time-invariant setting, the regret bound scales as~\cite{varici2022causal}
\begin{equation}
     \E[R(T)]\leq 2m+\tilde{\mcO}(\beta_T^{\prime L} d^{\frac{L}{2}}\sqrt{NT}) \ .
\end{equation}
If we directly apply the algorithm designed for the time-invariant setting to the model fluctuation setting with proper adjustments, we observed it would exhibit model deviation robustness but only for substantially small deviation levels. Specifically, to make the time-invariant algorithm robust, we need to adjust the term $\beta'_T$ so that it scales linearly with $C$ (see Appendix~\ref{proof:lem:naive} for details). This, in turn, induces a term  $C^L$ in the regret bound. Consequently, to preserve a sub-linear regret growth $T$, $\alpha$ need to fall in the interval $(0,\frac{1}{2L})$, which is a highly restrictive regime of model deviations, especially for graphs with long directed paths. For instance, under the DF measure, for $T=10^5$ and a graph with $L=3$, a node cannot be compromised more than $6$ samples to maintain sub-linear regret. {\bf This deviation level can be well below the noise level of the model, establishing that the algorithms designed for the time-invariant setting lack robustness.}

On the other hand, our \algonameRWLinSEMUCB~algorithm achieves a regret scaling of $T^{\max\{\frac{1}{2},\alpha\}}$, which preserves the optimal rate $\tilde\mcO(\sqrt{T})$ under the regime $\alpha\leq \frac{1}{2}$, which is independent of other parameters. Furthermore, a non-linear growth of regret in $T$ is achieved for deviation with $\alpha \in (\frac{1}{2},1)$, which is a significant improvement over the regime $(0,\frac{1}{2L})$. In the context of the earlier example with $T=10^5$, this indicates that the algorithm remains $\tilde\mcO(\sqrt{T})$ regret when the model faces $C\leq 316$ outlier samples and achieves sub-linear regret later on. The cost incurred for integrating robustness is reflected in the additional low-order terms we have in our regret bounds. We also remark that by setting $C=1$, our regret bound recovers the order of the time-invariant setting.

\section{Empirical Evaluations}
\label{sec:experiment}
In this section, we assess the robustness of the Robust-LCB algorithm and its scaling behavior with the graph parameters. To the best of our knowledge, there is no baseline CB algorithm that can be used as a natural baseline for performance comparisons. Furthermore, soft interventions on continuous variables of a CB model are implemented by only \algonameLinSEMUCB~of \cite{varici2022causal}. Therefore, to assess the robustness, we compare our \algonameRWLinSEMUCB~algorithm with \algonameLinSEMUCB~and the standard non-causal UCB algorithm.

\begin{figure}
    \centering
    \begin{minipage}{.3\textwidth}
        \centering
        \includegraphics[height=3.3 cm]{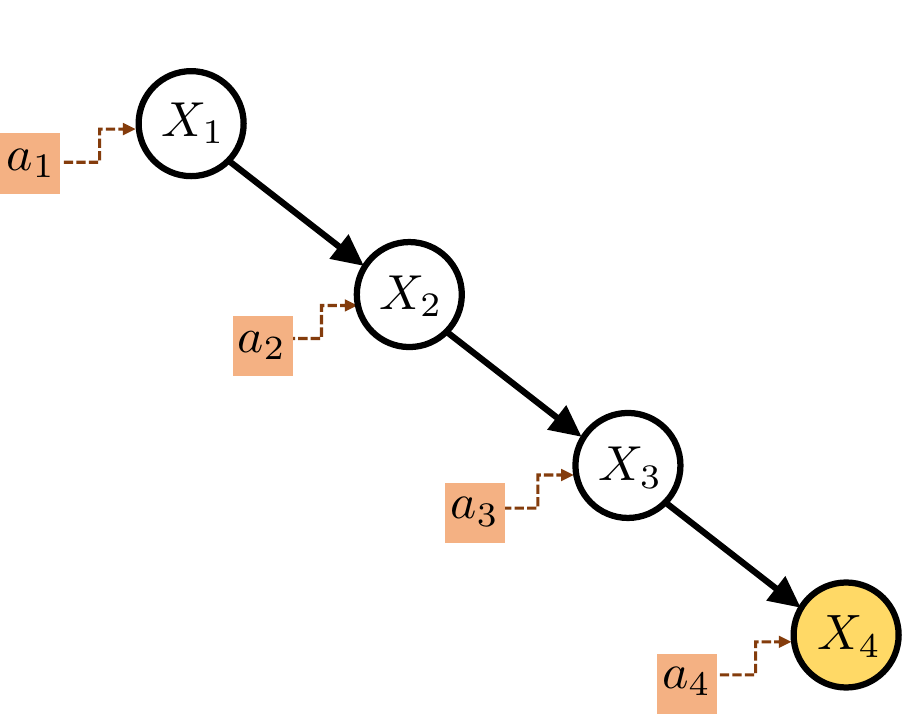}
        \caption{Chain graph with\\ $N=4$.}
        \label{fig:chain_example_add}
    \end{minipage}\hfill
    \begin{minipage}{0.3\textwidth}
        \centering
    \includegraphics[height=3.3 cm]{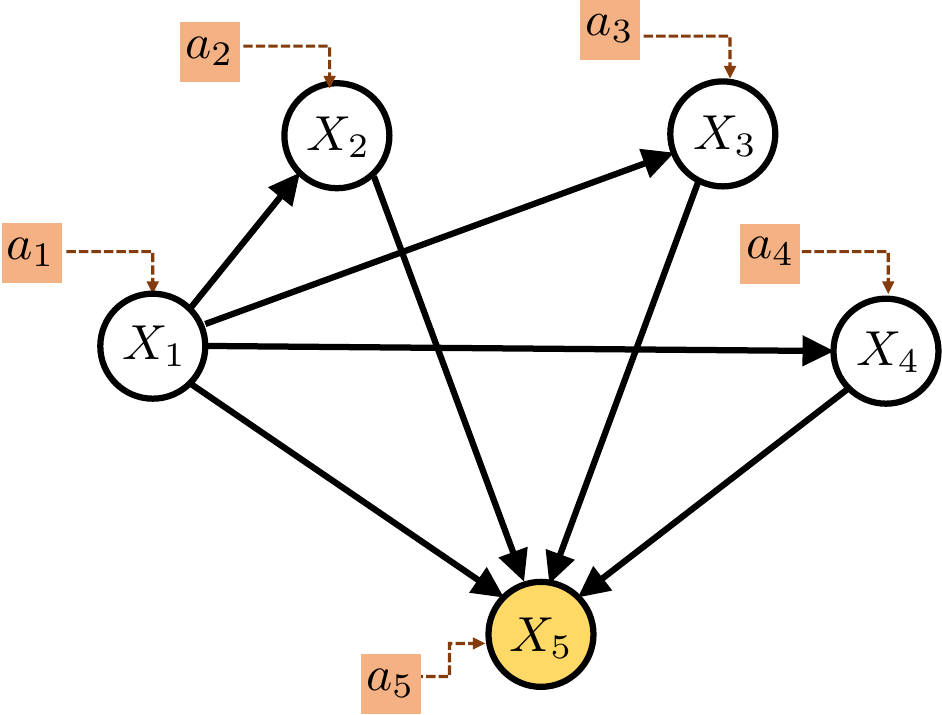}
        \caption{Confounded parallel graph  $N=5$.}
        \label{fig:confonded_example_add}
    \end{minipage}\hfill
    \begin{minipage}{.3\textwidth}
        \centering
        \includegraphics[height=3.3 cm]{figures/he_example.pdf}
        \caption{Hierarchical graph with $d=3$ and $L=2$.}
        \label{fig:he_example}
    \end{minipage}
\end{figure}

\vspace{.1 in}
\noindent\textbf{Parameter setting:} We consider three types of graph: chain graph (Figure~\ref{fig:chain_example_add}), confounded parallel graph (Figure~\ref{fig:confonded_example_add}) and hierarchical graph (Figure~\ref{fig:he_example}). The noise terms are uniformly sampled from $[0,2]$.
We consider the general setting that all nodes that have parents can suffer from model deviations. The norm of both observational weights $\norm{[\bB]_{i}}$ and the interventional weights $\norm{[\bB^*]_{i}}$ are set to $0.5$ and $1$ respectively. 
\begin{itemize}[leftmargin=0.15 in]
    \item \textbf{Chain graph.} The chain graph in Figure~\ref{fig:chain_example_add} is  a fundamental element of causal graphs. The observational weights $[\bB]_{i-1,i}$ and interventional weights $[\bB^*]_{i-1,i}$ for node $i\in[N]$ are set to $0.5$ and $1$, respectively.

    \item \textbf{Hierarchical graph.} For the experiments reported in Figure~\ref{fig:diff_d}, we use the graph structure with $L=2$ layers, and each layer has the same number of nodes $d \in \{1,2,3,4,5\}$. For the experiments illustrated in Figures~\ref{fig:he_regret} and \ref{fig:he_C}, we set the number of nodes in Layer 2 to $3$ and Layer 1 to $9$. We set the observational and interventional weights for node $i\in[N]$ to $0.5/\sqrt{|\Pa(i)|}$ and $1/\sqrt{|\Pa(i)|}$, respectively.
    
    \item \textbf{Confounded parallel graph.} The confounded parallel graph, as shown in Figure~\ref{fig:confonded_example_add}, is a mixture of the parallel graph and confounded graph in~\cite{lattimore2016causal}, where node $1$ is the parent of all other nodes, and the reward node $N$ is the child of all other nodes. We set the observational and the interventional weights for nodes $i \in \{2,\dots,N-1\}$ to $0.5$ and $1$, respectively. For the reward node $N$, its parents' observational and interventional weights are set to $0.5/\sqrt{N-1}$ and $1/\sqrt{N-1}$, respectively.
\end{itemize}
We let the deviations on the model occur at earlier rounds to simulate the worst-case scenario for a given deviation level $C$. 
When a deviation occurs on node $i\in[N]$, the weights are deliberately altered to change the optimal action, thereby challenging the algorithm's performance. The simulations are repeated $100$ times, and the average cumulative regret is reported.

\begin{figure}
    \centering
    \includegraphics[height=5 cm]{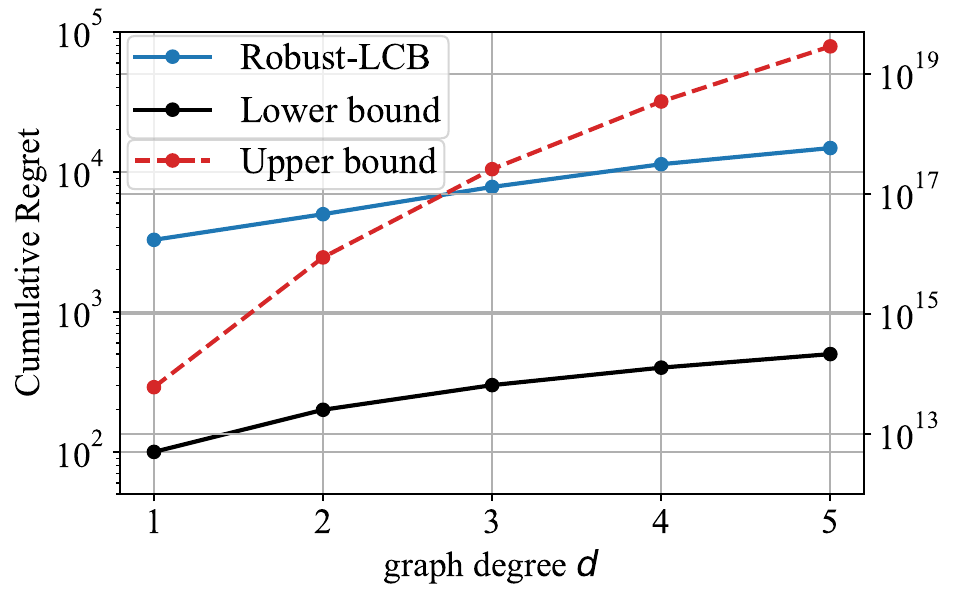}
    \caption{Cumulative regret of \algonameRWLinSEMUCB\ with different graph degree $d$ and $L=2$, $N=Ld+1$.}
    \label{fig:diff_d}
\end{figure}
\vspace{.1 in}
\noindent\textbf{Scaling behavior with degree $d$.} Figure~\ref{fig:diff_d} illustrates the variations of the cumulative regret $\E \left[R(T)\right]$ versus the graph degree $d$ when $T=40000$. We compare the regret of the Robust-LCB algorithm (blue curve with its scale on the left axis), the lower bound characterized in Section~\ref{sec:lowerbound} (black curve with its scale on the left axis), and the upper bound specified in~\eqref{equ:upperbound} (red curve with its scale on the right axis). All three curves suggest a polynomial behavior in $d$, which conforms to our theoretical results. 

\begin{figure}[ht]
    \centering
        \begin{subfigure}{0.32\textwidth}
        \centering
        \includegraphics[height=3.3 cm]{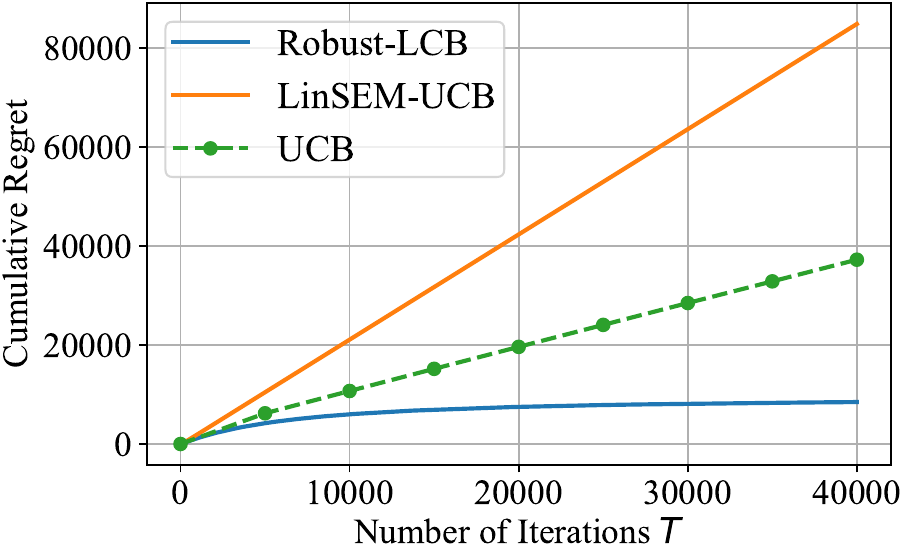}
        \caption{Chain graph.}
    \label{fig:chain_regret_add}
    \end{subfigure}
    \begin{subfigure}{0.32\textwidth}
        \centering
        \includegraphics[height=3.3 cm]{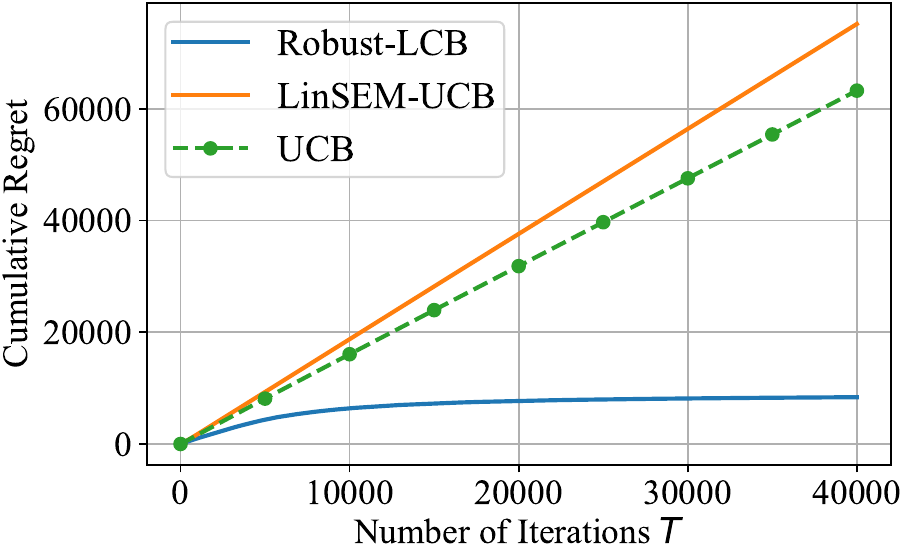}
        \caption{Confounded parallel graph.}
    \label{fig:confounded_regret_add}
    \end{subfigure}
    \begin{subfigure}{0.32\textwidth}
        \centering
        \includegraphics[height = 3.3 cm]{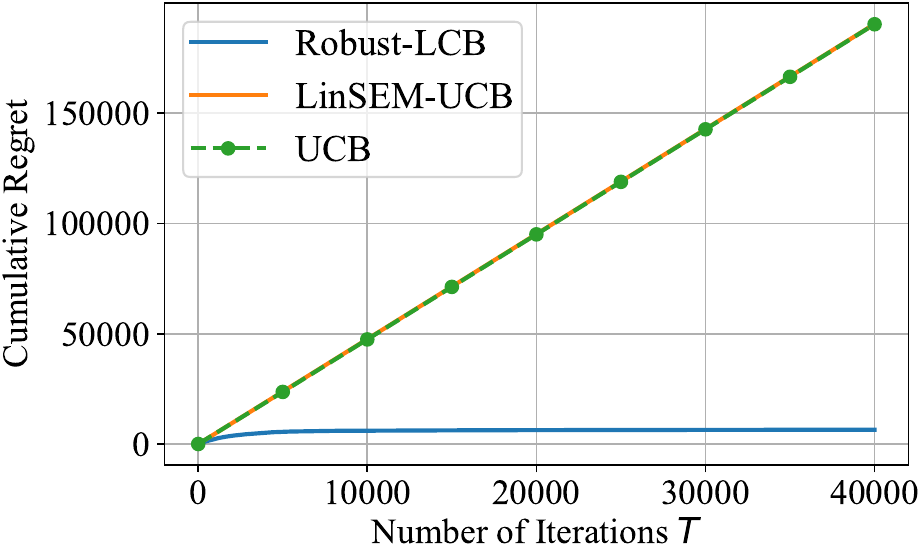}
        \caption{Hierarchical graph}
        \label{fig:he_regret}
    \end{subfigure}
    \caption{Cumulative regret when $C=\sqrt{T}$
        for different algorithms.}
    \label{fig:regret}
\end{figure}

\vspace{.1 in}
\noindent\textbf{Comparison of the bounds.} Figure~\ref{fig:regret} compares the cumulative regret of \algonameRWLinSEMUCB\ with that of \algonameLinSEMUCB\ and UCB under a model deviation level of $C =T^{0.5}$. It demonstrates that only \algonameRWLinSEMUCB\  achieves sub-linear regret, whereas the other two algorithms incur linear regret with respect to the horizon $T$. Furthermore, it is noteworthy that the LinSEM-UCB always exhibits nearly the worst possible regret as the design of the deviation also showcases the worst case for LinSEM-UCB. In contrast, the regret of UCB tends to be the worst possible outcome as the graph's complexity increases even when the deviation is not designed for it. Nevertheless, these findings imply that the estimators of these algorithms become ineffective when faced with such deviations, resulting in the selection of a sub-optimal (possibly the worst) arm.

\begin{figure}
       \begin{subfigure}{0.33\textwidth}
        \centering
        \includegraphics[height=3.2 cm]{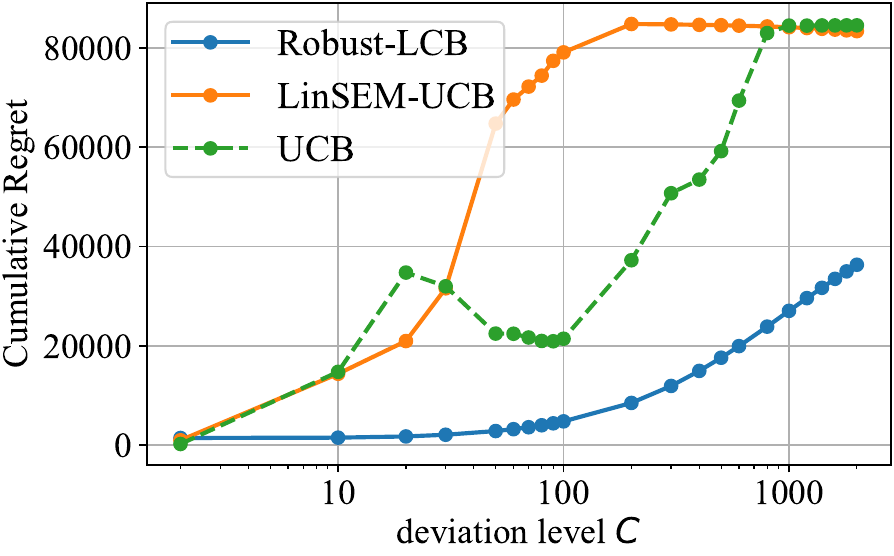}
        \caption{Chain graph.}
        \label{fig:chain_C_add}
    \end{subfigure}
    \begin{subfigure}{0.33\textwidth}
        \centering
\includegraphics[height=3.2 cm]{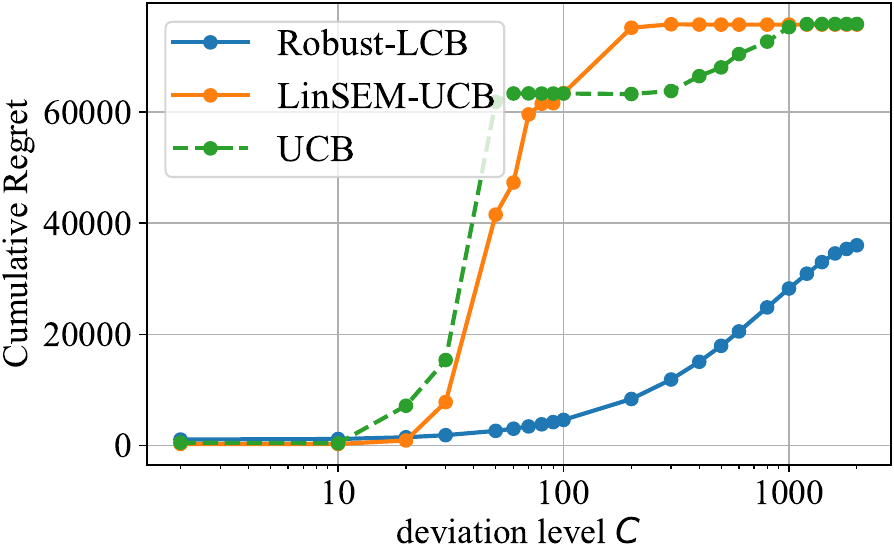}
        \caption{Confounded parallel graph.}
\label{fig:confounded_C_add}
    \end{subfigure}
    \begin{subfigure}{0.32\textwidth}
        \centering
        \includegraphics[height=3.3 cm]{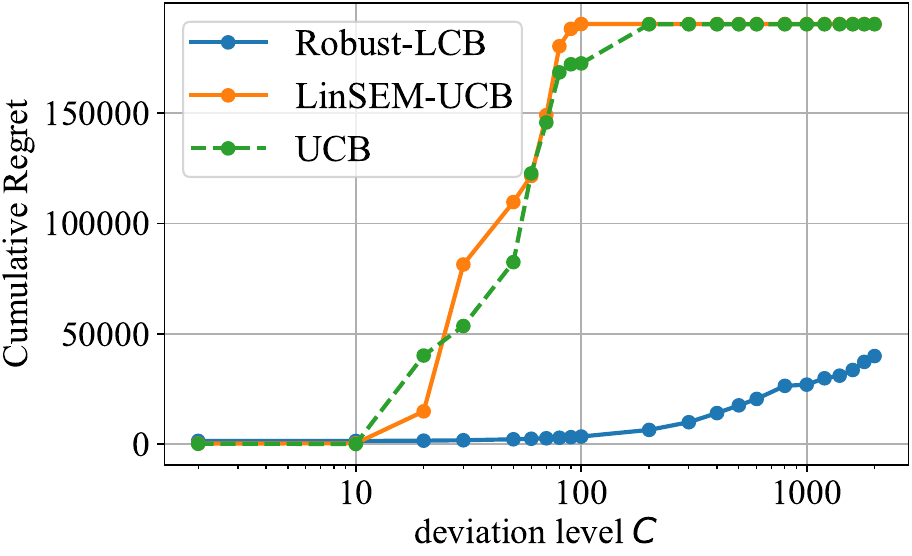}
        \caption{Hierarchical graph.}
        \label{fig:he_C}
\end{subfigure}
    \caption{Cumulative regret at $T = 40000$ for different deviation levels $C$.}
    \label{fig:diff_c}
\end{figure}

\vspace{.1 in}
\noindent\textbf{Robustness against $C$.} Figure~\ref{fig:diff_c} plots the cumulative regret at $T=40000$ when the model deviation budget $C$ changes from $2$ to $2000$. 
We observe that \algonameRWLinSEMUCB\  and UCB perform slightly better than our algorithm when there is almost no deviation. This can be viewed as the compromise needed to guarantee the robustness of our algorithm. However, both \algonameRWLinSEMUCB\ and UCB begin to fail with even a minor model deviation, as small as $C= 15$ for all three scenarios. Furthermore, they tend to reach nearly the worst possible regret when the deviation level rises high enough (e.g., 200 for the hierarchical graph). In the chain graph scenario, the UCB algorithm exhibits moderate regret when $C\leq 300$. However, as demonstrated in Figure~\ref{fig:chain_regret_add}, it mistakenly treats a non-optimal arm as the optimal one, resulting in linear regret, which happens when $C\geq 10$. In comparison, \algonameRWLinSEMUCB\ outperforms when the deviation is more than negligible, as its regret scales sub-linear with the deviation level $C$.


\section{Concluding Remarks}
In this paper, we have studied the sequential design of interventions over graphical causal models where both the observational and interventional models are unknown and undergo temporal variations. We have considered the general soft intervention model and have designed an algorithm for identifying the intervention mechanism that optimizes a utility function over the graph while exhibiting robustness against model variations. This objective has been naturally cast and analyzed as a causal bandit problem. We have focused on causal structures described by linear structural equation models. We have analyzed the proposed algorithm from a cumulative regret perspective, where we have characterized the dependence of the regret on the graph parameters. The main observations are the following. (i) The proposed algorithm maintains sub-linear or nearly optimal regret under a wide range of relevant model deviation measures. This contrasts sharply with the existing algorithms designed for fixed models, which lose their sub-linear regret rate with a minimum level of model variations. (ii) While the cardinality of the intervention space grows exponentially with the graph size, our regret maintains a linear growth in the graph size. (iii) The regret bound depends on the graph structure only through its parameters (maximum degree and the length of the longest causal path). Finally, we have established an information-theoretic lower bound to demonstrate the tightness of our upper bound.

We conclude by providing some potential future directions. The first direction is tightening the gap between the graph-dependent parameters in the upper and lower bounds. Secondly, an alternative model to consider is the setting is the case of non-stationary bandits, in which the temporal variations can be potentially permanent. In such settings, the relevant metric to analyze is the dynamic regret specifically designed for non-stationary bandits. This metric becomes relevant when assuming the absence of a nominal model and the permanent effects of deviations. Lastly, we expect that our insights in this paper, including the effectiveness of weighted exploration bonus, can be extended into more general causal bandit settings, as well as other structured bandit settings involving model deviation. Examples include extending it to causal bandits with general link functions and combinatorial bandits.

\appendix

\section{Additional Notations}
First, we provide notations that are useful in our analyses. Since we are dealing with matrices, we denote the singular values of a matrix $\bA\in\R^{M \times N}$, where $M\geq N$, by
\begin{align}
    \sigma_{1}(\bA) &\geq \sigma_{2}(\bA) \geq \dots \geq \sigma_{N}(\bA) \ .
\end{align}
In the proof, we often work with zero-padded vectors and corresponding matrices. As a result, the matrices that contain these vectors have non-trivial \emph{null space} leading to zero singular values. In such cases, we use the \emph{effective} smallest singular value that is non-zero. We denote the \emph{effective} largest and smallest eigenvalues that correspond to effective dimensions of a positive semidefinite matrix $\bA$ with rank $k$ by
\begin{equation}
    \smax{\bA} \triangleq \sigma_{1}(\bA) \ , \quad \mbox{and} \quad \smin{\bA} \triangleq \sigma_{k}(\bA) \ .
\end{equation}
For a square matrix $\bU = \bA \bA^{\top} \in \mathbb{R}^{N \times N}$, we denote the \emph{effective} largest and smallest eigenvalues by\footnote{For matrix $\bV=\bU+\bI$, we denote the \emph{effective} smallest eigenvalues by $\lmin{\bV} \triangleq \sigma_{\min}^2(\bA) +1$.}
\begin{equation}
    \lmax{\bU} \triangleq \lmax{\bA \bA^{\top}} 
    = \sigma_{\max}^2(\bA) \ , \quad 
    \mbox{and} \quad \lmin{\bU} \triangleq \lmin{\bA \bA^{\top}} 
    = \sigma_{\min}^2(\bA) \ . 
\end{equation}
Then we construct data matrices that are highly related to Gram matrices. At time~$t\in\N$ and for any node $i\in[N]$, the data matrices $\bU_{i}(t)\in\R^{t \times N}$ and $\bU^{*}_{i}(t)\in\R^{t \times N}$ consist of the weighted observational and interventional data, respectively. Specifically, for any $s\in[t]$ and $i\in[N]$, we define
\begin{equation}\label{eq:D_it_obs}
    \big[\bU_{i}^{\top}(t)\big]_{s} \triangleq \mathbbm{1}{\{i \notin a(s)\}} \sqrt{w_i(s)} \bX_{\Pa(i)}^{\top}(s) \ ,  \quad 
\mbox{and} \quad     \big[{\bU^{*}_{i}}^{\top}(t)\big]_{s} \triangleq \mathbbm{1}{\{i \in a(s)\}} \sqrt{w_i(s)}\bX_{\Pa(i)}^{\top}(s) \ . 
\end{equation}
Similarly to \eqref{eq:Ba_construct}, we denote the relevant data matrices for node $i\in[N]$ under intervention $a \in \mcA$ by
\begin{align}
    \bU_{i,a}(t) &\; \triangleq \; \mathbbm{1}{\{i \in a\}} \bU^{*}_{i}(t) +  \mathbbm{1}{\{i \notin a\}} \bU_{i}(t) \ ,  \label{eq:D_ita}\\
    \bV_{i,a}(t) &\;\triangleq\; \mathbbm{1}{\{i \in a\}} \bV^{*}_{i}(t) +  \mathbbm{1}{\{i \notin a\}} \bV_{i}(t) \ . \label{eq:V_ita}
\end{align}
Combining~\eqref{eq:D_it_obs} and~\eqref{eq:V_ita}, we have
\begin{align}
    \bV_{i,a}(t) &= \bU_{i,a}^{\top}(t) \bU_{i,a}(t) + \bI_{N} \ . \label{eq:V_ita_D_ita}
\end{align}
Similarly we define the data matrices that are related to $\widetilde{\bV}_{i,a}(t)$ as
\begin{equation}\label{eq:D_it_obs_tilde}
    \big[\widetilde{\bU}_{i}^{\top}(t)\big]_{s} \;\triangleq\; \mathbbm{1}{\{i \notin a(s)\}} w_i(s) \bX_{\Pa(i)}^{\top}(s) \ ,  \quad 
\mbox{and} \quad     \big[{{\widetilde{\bU}}^{*}_{i}}{}^{\top}(t)\big]_{s} \; \triangleq \; \mathbbm{1}{\{i \in a(s)\}} w_i(s)\bX_{\Pa(i)}^{\top}(s) \ .  
\end{equation}
The relevant data matrices for node $i\in[N]$ under intervention $a \in \mcA$ are
\begin{equation}
    \label{eq:D_ita_tilde}
    \widetilde{\bU}_{i,a}(t) \;\triangleq\; \mathbbm{1}{\{i \in A\}} \bU^{*}_{i}(t) +  \mathbbm{1}{\{i \notin A\}} \bU_{i}(t) \ ,  \quad \mbox{and} \quad 
    \widetilde{\bV}_{i,a}(t) = \widetilde{\bU}_{i,a}^{\top}(t) \widetilde{\bU}_{i,a}(t) + \bI_{N} \ .
\end{equation}
Define $N^{*}_{i}(t)$ as the number of times that node $i\in[N]$ is intervened, and $N_i(t)$ as its complement, i.e.,
\begin{align} \label{eq:def_N_it}
    N^{*}_{i}(t) \;\triangleq\; \sum_{s=1}^t \mathbbm{1}{\{i \in a(s)\}}\ , \quad \mbox{and} \quad 
    N_{i}(t) \;\triangleq\; t -  N^{*}_{i}(t)\ . 
\end{align}
Accordingly, for any $i\in[N]$ and $t\in\N$, define
\begin{align} \label{eq:def_N_iat}
    N_{i,a}(t) &\triangleq \mathbbm{1}\{i \in a\} N^{*}_{i}(t) + \mathbbm{1}\{i \notin a\} N_{i}(t) \ , 
\end{align}
Furthermore, we define the error and its $l$-th power of the estimator of $\bB$ matrices as
\begin{align}
    \Delta_{a}(t) & \triangleq \bB_{a}(t) - \bB_{a}, \ \forall i \in [N]\ , \quad \mbox{and} \quad \Delta_{a}^{(\ell)}(t)  \triangleq \bB_{a}^{\ell}(t) - \bB_{a}^{\ell} \label{eq:def_delta_alit} \ .
\end{align}

\section{Proof of Lemma~\ref{lm:bound_l_paths}}
\label{proof:lm:bound_l_paths}

\begin{proof} When it is clear from context, we use the shorthand terms $\mcB$ for $\mcB_a$ in this proof.
According to the given definitions 
$\Delta_{\bA}^{(\ell)} = [\Delta_{\bA} + \bB]^{\ell} - \bB^{\ell}$, each term in the binomial expansion of $\Delta_{\bA}^{(\ell)}$ can be represented as a product involving factors of  $\Delta_{\bA}$ and $\bB$. For any $\ell \in [L]$ and $k \in [\ell] \cup \{0\}$, there exist $\binom{\ell}{k}$ terms that consist of the $\Delta_{\bA}$ factor appearing $k$ times and the $\bB$ factor appearing$(\ell-k)$ times. We denote the set of these product terms by
\begin{align}
    \mcH_{\ell,k} \triangleq \{ \bH : \text{$\bH$ has $\Delta_{\bA}$ factor $k$ times and $\bB$ factor $\ell-k$ times} \} \ .
\end{align}
Therefore, sets $\mcH_{\ell,1},\dots,\mcH_{\ell,\ell}$ contain all valid products composed of $\bB$ and $\Delta_{\bA}$. Hence, we can write the expansion of $\Delta_{\bA}^{(\ell)}$ as
\begin{align}
    \Delta_{\bA}^{(\ell)} = \sum_{k=1}^{\ell} \sum_{\bH \in \mcH_{\ell,k}} \bH \label{eq:delta_expansion_sets} \ .
\end{align}
To bound the norm of $\Delta_{\bA}^{(\ell)}$, we first bound the norm of each element in the summation. We show by induction that for any $\ell\in [L]$, $k \in [\ell] \cup \{0\}$, and $\bH\in\mcH_{\ell,k}$,
\begin{align}
    \norm{[\bH]_{i}} \leq  d^{\frac{\ell-1}{2}} \beta^k \lambda^k\ , \quad  \forall i \in [N]\ , \quad \mbox{where} \quad  \lambda \triangleq \max_{i \in [N]} \lminn{\bM_{i}}{-1/2}  \ .\quad \label{eq:bound_lk}
\end{align}
Consider $\ell=1$. For $k=0$, we have $\mcH_{1,0}=\{\bB\}$ and $\norm{[\bB]_{i}}\leq 1$. For $k=1$, $\mcH_{1,1} = \{\Delta_{\bA}\}$, and 
\begin{align}
\norm{[\Delta_{\bA}]_{i}}&\leq \norm{[\Delta_{\bA}]_{i} \odot \mathbf{1}(\Pa(i))}_{\bM_{i}} \lminn{\bM_{i}}{-1/2} = \norm{[\Delta_{\bA}]_{i}}_{\bM_i} \lminn{\bM_{i}}{-1/2} 
\leq \beta \lambda\ .
\end{align}
Therefore, \eqref{eq:bound_lk} holds for $\ell=1$. Now assume that~\eqref{eq:bound_lk} holds for every $1,\dots,\ell-1$, for some $\ell \geq 2$. Consider a product term $H \in \mcH_{\ell,k}$, for some $ k \in [\ell] \cup \{0\}$. The first factor of $\bH$ is either $\bB_{a}$ or $\Delta_{\bA}$, and we analyze the induction step for each of these possibilities separately. 

\noindent {\bf Case 1.}~ If $\bH$ starts with $\bB$, represent it by $\bH = \bB {\bar \bH}$, where $\bar \bH \in \mcH_{\ell-1,k}$ and $k \in [\ell-1] \cup \{0\}$. Using the induction assumption for the elements of set $\mcH_{\ell-1,k}$ we obtain
    \begin{align}
        \norm{[\bH]_i}^2 &= \norm{(\bB \bar \bH)_i}^2  =  \sum_{u,v \in \Pa(i)}[\bB]_{u,i}[\bB]_{v,i} {\bar \bH}_{v}^{\top} {\bar \bH}_{u}  \\
        &\leq d \sum_{u \in \Pa(i)} ([\bB]_{u,i})^2 \norm{{\bar \bH}_{u}}^2 \label{equ:bound_m1} \leq d^{\ell-1}  \beta^{2k} \lambda^{2}  \sum_{u \in \Pa(i)} ([\bB]_{u,i})^2    = d^{\ell-1}  \beta^{2k} \lambda^{2k} \underset{\leq 1}{\underbrace{\norm{[\bB]_{i}}^2}}   \\
        &\leq  d^{\ell-1} \beta^{2k} \lambda^{2k} \ , \label{eq:bound_lk_induction_case1}
    \end{align}
where in \eqref{equ:bound_m1}  we use Cauchy–Schwarz inequality and the inductive hypothesis~\eqref{eq:bound_lk}.

\noindent {\bf Case 2.}.~ If $\bH$ starts with $\Delta_{\bA}$ represent it by $\bH = \Delta_{\bA} {\bar \bH}$, where $\bar \bH \in \mcH_{\ell-1,k-1}$ and $k \in [\ell]$. Similarly to the first case, we have
    \begin{align}
        \norm{[\bH]_i}^2 &= \norm{[\Delta_{\bA} \bar \bH]_{i}}^2 =  \sum_{u,v \in \Pa(i)} [\Delta_{\bA}]_{u,i} [\Delta_{\bA}]_{v,i} [\bar \bH]_{v}^{\top} [\bar \bH]_{u}  \\
        &\leq d \sum_{u \in \Pa(i)} ([\Delta_{\bA}]_{u,i})^2 \norm{[\bar \bH]_{u}}^2\label{lem:midd}  \leq d^{\ell-1}   \beta^{2k-2}  \lambda^{2k-2} \sum_{u \in \Pa(i)} ([\Delta_{\bA}]_{u,i})^2  \\
        &= d^{\ell-1} \beta^{2k-2} \lambda^{2k-2}  \norm{[\Delta_{\bA}]_{i}}^2     \leq d^{\ell-1} \beta^{2k-2} \lambda^{2k-2}  \beta^2 \underset{\leq \lambda^{2}}{\underbrace{\lminn{\bM_{i}}{-1}}} \label{eq:bound_lk_induction_case2_mid}   \\
        &\leq d^{\ell-1} \beta^{2k} \lambda^{2k} \ . \label{eq:bound_lk_induction_case2} 
    \end{align}
where \eqref{lem:midd} we use Cauchy-Schwarz inequality, and in \eqref{eq:bound_lk_induction_case2_mid} we use the fact that  $[\Delta_{\bA}]_{i} = [\Delta_{\bA}]_{i} \odot \mathbf{1}(\Pa(i))$ and $\norm{[\Delta_{\bA}]_{i}} \leq \norm{[\Delta_{\bA}]_{i}}_{\bM_{i}} \lminn{\bM_{i}}{-1/2} $.
Taking the square-roots of both sides in \eqref{eq:bound_lk_induction_case1} and \eqref{eq:bound_lk_induction_case2} yields
\begin{align}
     \norm{[\bH]_i}&\leq  d^{\frac{\ell-1}{2}}  \beta^{k} \lambda^k \ , \label{eq:bound_lk_final}
\end{align}
which is the desired inequality for all $k \in [\ell] \cup \{0\}$. This completes the proof of~\eqref{eq:bound_lk} by induction. The final result follows by applying  \eqref{eq:bound_lk_final} to each term in the sum \eqref{eq:delta_expansion_sets}.
\begin{align}
    \norm{\big[\Delta_{\bA}^{(\ell)}\big]_{i}} &= \Bigg\|\sum_{k=1}^{\ell} \sum_{\bH \in \mcH_{\ell,k}} [\bH]_i\Bigg\| \\
    &\leq \sum_{k=1}^{\ell} \sum_{\bH \in \mcH_{\ell,k}} \norm{[\bH]_i}  \leq d^{\frac{\ell-1}{2}}  \sum_{k=1}^{\ell} |\mcH_{\ell,k}|  \beta^{k} \lambda^{k}  = d^{\frac{\ell-1}{2}} \sum_{k=1}^{\ell} \binom{\ell}{k}  \beta^{k}  \lambda^{k}  \\
    &< d^{\frac{\ell-1}{2}} (\beta+1)^{\ell} \lambda \label{equ:m3}  = d^{\frac{\ell-1}{2}} (\beta+1)^{\ell} \max_{i \in [N]} \lminn{\bM_{i}}{-1/2} \ ,
\end{align}
where \eqref{equ:m3} is due to the binomial expansion of $(\beta+1)^{\ell}$ and $\bM_i \succeq \bI, \ \forall i\in[N]$ such that $\lambda\leq 1$.
\end{proof}

\section{Proof of Theorem~\ref{thm:measure2}} 
\label{proof:thm:measure2}
The proof mechanism follows the same line of arguments as~\cite[Theorem 5]{varici2022causal} but with major distinctions. Firstly, we provide a new lemma on the bound on the power of estimation error $\Delta_{a}^{(\ell)}(t) \triangleq \bB_{a}^{\ell}(t-1) - \bB_{a}^{\ell}$,  Furthermore, the effect of the weighted ridge regression, distinct confidence ellipsoids, and the definition of error events, namely $\widetilde{\mcE}_{i,n}(t)$ and $\widetilde{\mcE}^*_{i,n}(t)$ are investigated. Finally, we bound a new function of eigenvalues of weighted Gram matrices  $\bV_{i,a(t)}(t)$ and $\widetilde{\bV}_{i,a(t)}$. To begin with, we first state the lemma which upper bounds the estimation error.

Based on the above estimation error, we start the proof by decomposing the regret in \eqref{equ:regret} as
\begin{align}
    \E[R(T)]  &= \E\left[ \sum_{t=1}^{T} \big(\mu_{a^*} - \mu_{a(t)}\big) \right]\\
    &= \E\left[\sum_{t=1}^T \inner{f(\bB_{a^*}) -f(\bB_{a(t)})}{\bepsilon(t)} \right]\\
    &= \E\left[\sum_{t=1}^T \inner{f(\bB_{a^*}) - f(\bB_{a(t)})}{\bnu}\right]   \ ,
\end{align}
where the second equation is a result of~\cite[Lemma~1]{varici2022causal}, which is stated as Lemma~\ref{lem:expectedvalue}. The last equation is due to the inner product being a linear function and $\E [\bepsilon]=\bnu $.

Now define the error events $\mcE_{i}$ and $ \mcE^{*}_{i}$ for $i\in[N]$ for each estimator 
\begin{align}
    \mcE_{i} &\triangleq \biggl\{ \forall t \in [T] : \norm{[\bB(t-1)]_{i}-[\bB]_{i}}_{\bV_i(t) [\widetilde{\bV}_i(t)]^{-1} \bV_i(t)} \leq \beta_t \biggr\}  \ ,  \\
   \mbox{and} \quad  \mcE^{*}_{i} &\triangleq \biggl\{ \forall t \in [T] : \norm{[\bB^{*}(t-1)]_{i}-[\bB^{*}]_{i}}_{\bV^{*}_i(t) [\widetilde{\bV}^{*}_i(t)]^{-1} \bV^{*}_i(t)} \leq \beta_t \biggr\}  \ ,
\end{align}
where the $\beta_t$ is chosen as
\begin{equation}
\beta_t\;\triangleq\;\sqrt{2\log\left(2NT\right)+d\log\left(1+m^2t/dC^2\right)} + 1 + m\ ,
\end{equation}
Let $\mcE_{\cap}$ denote the event that all of the events $\{\mcE_{i}, \mcE^{*}_{i}: i \in [N]\}$ occur simultaneously, i.e., 
\begin{align}
    \mcE_{\cap} \triangleq \Big(\bigcap_{i=1}^N  \mcE_{i} \Big) \bigcap \Big(\bigcap_{i=1}^N  \mcE^{*}_{i} \Big)  \ .  \label{eq:ucb_conf_interval_event_union_RW}
\end{align}
Then by Lemma~\ref{lem:beta_Tindeviation}, we have
 \begin{align}
    \P \left( \mcE_{\cap}^{\C} \right) &\leq  \sum_{i=1}^{N} \P(\mcE_{i}^{\C}) + \sum_{i=1}^{N} \P({\mcE^{*}_{i}}^{\C})  
     \leq \sum_{i=1}^N \left(\frac{1}{2NT} + \frac{1}{2NT} \right)  = \frac{1}{T} \ . \label{eq:bound_Ec_RW}
\end{align}
Then we can bound the regret as follows.
\begin{align} 
R(T) &\leq  \E\left[\mathbbm{1}\{\mcE_{\cap}^{\C}\} \sum_{t=1}^{T} \inner{f(\bB_{a^*}) - f(\bB_{a(t)})}{\bnu} \right] + \E\left[\mathbbm{1}\{\mcE_{\cap}\} \sum_{t=1}^{T} r(t) \right] \\
    &\leq 2m T \P(\mcE_{\cap}^{\C}) +\E\left[\mathbbm{1}\{\mcE_{\cap}\} \sum_{t=1}^{T} r(t) \right] \\
    &\leq  2m + \E\left[\mathbbm{1}\{\mcE_{\cap}\} \sum_{t=1}^T \inner{f(\bB_{a^*})-f(\bB_{a(t)})}{\bnu}  \right] \ . \label{eq:ucb_proof_mid1_RW}
\end{align}
According to the arm selection rules, we have ${\rm UCB}_{a^*}(t) \leq {\rm UCB}_{a(t)}(t)$. Let $\widetilde \bB_{a}$ be the matrix that attains ${\rm UCB}_{a}(t)$, i.e., 
\begin{equation}
    \widetilde \bB_{a}=\argmax_{[\Theta]_i\in \mcC_{i,a}(t)}\inner{f(\Theta)}{\bnu }\ .
\end{equation}
Then under the event $\mcE_{\cap}$, we have
\begin{align}
\inner{f(\bB_{a^*})-f(\bB_{a(t)})}{\bnu } &\leq{\rm UCB}_{a^*}(t) - \inner{f(\bB_{a(t)})}{\bnu } \\
&\leq{\rm UCB}_{a(t)}(t) - \inner{f(\bB_{a(t)})}{\bnu } \\
&= \langle f(\widetilde \bB_{a(t)})-f(\bB_{a(t)}), \bnu \rangle  \ . \label{eq:ucb_proof_mid2_RW}
\end{align}
Subsequently, we use \eqref{eq:ucb_proof_mid2_RW} to derive an upper bound for the expectation term in  \eqref{eq:ucb_proof_mid1_RW}, thereby eliminating the dependence on $a^*$.
\begin{align}
\hspace{-0.12 in}\E\left[\mathbbm{1}\{\mcE_{\cap}\}  \sum_{t=1}^T \inner{f(\bB_{a^*})  -f(\bB_{a(t)})}{\!\bnu}  \right] &\leq \E\left[\mathbbm{1}\{\mcE_{\cap}\} \sum_{t=1}^T \inner{f(\widetilde \bB_{a(t)})-f(\bB_{a(t)})}{\!\bnu}  \right] \label{eq:ucb_refer_in_ts} \\
    &\leq \norm{\bnu} \E\left[\mathbbm{1}\{\mcE_{\cap}\} \sum_{t=1}^T \norm{f(\widetilde \bB_{a(t)})-f(\bB_{a(t)})} \right]  \label{eq:ucb_proof_mid4_RW}\\
    &\leq \norm{\bnu} \E\left[\mathbbm{1}\{\mcE_{\cap}\} \sum_{t=1}^T \norm{f(\widetilde \bB_{a(t)})-f(\bB_{a(t)}(t))} \right]\nonumber\\
    &\quad + \norm{\bnu} \E\left[\mathbbm{1}\{\mcE_{\cap}\} \sum_{t=1}^T \norm{f(\widetilde \bB_{a(t)}(t))-f(\bB_{a(t)})} \right]\ .\label{eq:ucb_proof_mid5_RW}
\end{align}
Note that \eqref{eq:ucb_proof_mid4_RW} follows from the Cauchy-Schwarz inequality while \eqref{eq:ucb_proof_mid5_RW} is due to the triangle inequality. Next, we examine the norm term within the expectation expression in \eqref{eq:ucb_proof_mid5_RW}. By applying the definition of $f$, we obtain
\begin{align}
    \norm{f(\widetilde \bB_{a(t)})-f(\bB_{a(t)}(t))} 
    &\leq \sum_{\ell=1}^L \norm{\big[\widetilde \bB_{a(t)}^{\ell}\big]_{N} - \big[\bB_{a(t)}^{\ell}(t)\big]_{N}} \ ,\\
    \norm{f(\widetilde \bB_{a(t)}(t))-f(\bB_{a(t)})} &\leq \sum_{\ell=1}^L\norm{\big[\bB_{a(t)}^{\ell}(t)\big]_{N} - \big[\bB_{a(t)}^{\ell}\big]_{N}}  \ . \label{eq:ucb_proof_mid5_2}
\end{align}
By the definition of Gram matrices, we have
\begin{align}
    \widetilde{\bV}_{i,a(t)}(t) \preceq\ \bV_{i,a(t)}(t) \ , \quad \bV_{i,a(t)}(t) [\widetilde{\bV}_{i,a(t)}(t)]^{-1} \bV_{i,a(t)}(t) \succeq \bV_{i,a(t)}(t) \succeq \bI_N\ .
\end{align}
We consider the following tuple of matrices
\begin{align}
    &\left(\widetilde{\bB}_{a(t)}, \bB_{a(t)}(t), \bV_{i,a(t)}(t) [\widetilde{\bV}_{i,a(t)}(t)]^{-1} \bV_{i,a(t)}(t), \  i\in[N]\right)\ , \\
    \mbox{and} \quad &\left(\bB_{a(t)}(t), \bB_{a(t)}, \bV_{i,a(t)}(t) [\widetilde{\bV}_{i,a(t)}(t)]^{-1} \bV_{i,a(t)}(t), \  i\in[N] \right) \ .
\end{align}
These tuple of matrices satisfy the condition of Lemma~\ref{lm:bound_l_paths}. We further define the maximum confidence radius $\beta_T\;\triangleq\;\sqrt{2\log\left(2NT\right)+d\log\left(1+m^2T/dC^2\right)} + 1 + m$ and
\begin{align}
    \lambda(t) \triangleq\max_{i \in [N]}  \frac{\sqrt{\lmax{\widetilde{\bV}_{i,a(t)}(t)}}}{\lmin{\bV_{i,a(t)}(t)}} \ , \qquad
    \lambda_T \triangleq \E \left[\sum_{t=1}^{T}\lambda(t )\right] \ .
\end{align}
By using Lemma~\ref{lm:bound_l_paths} to upper bound the tow terms in \eqref{eq:ucb_proof_mid5_2}, we have
\begin{align}
    &\E\left[\mathbbm{1}\{\mcE_{\cap}\}  \sum_{t=1}^T \inner{f(\bB_{a^*})  -f(\bB_{a(t)})}{\bnu}  \right] \\
     &< 2\sum_{\ell=1}^L d^{\frac{\ell-1}{2}} \big(\beta_t +1\big)^\ell \E\left[\mathbbm{1}\{\mcE_{\cap}\} \sum_{t=1}^T \max_{i \in [N]}  \lminn{\bV_{i,a(t)}(t) [\widetilde{\bV}_{i,a(t)}(t)]^{-1} \bV_{i,a(t)}(t)}{-1/2}  \right]  \label{eq:ucb_proof_mid9} \\
     &\leq 2\sum_{\ell=1}^L d^{\frac{\ell-1}{2}} \big(\beta_T +1\big)^\ell \E\left[\mathbbm{1}\{\mcE_{\cap}\} \sum_{t=1}^T \lambda(t) \right]  \\
     & \leq 2 \lambda_T \sum_{\ell=1}^L d^{\frac{\ell-1}{2}} \big(\beta_T +1\big)^\ell \\
     &\leq 4 \lambda_T \big(\beta_T +1\big)^L d^{\frac{L-1}{2}}\ ,  \label{eq:ucb_proof_mid11_RW}
\end{align}
where in \eqref{eq:ucb_proof_mid11_RW}, we use the fact that $\sum_{i=1}^{L}q^L\leq 2q^L$ for $q\geq2$ and $\sqrt{d}(\beta_T+1)>2$.

\paragraph{Bounding $\E\left[\sum_{i=1}^T \lambda(t)\right]$}
What remains is to bound the term $\E\left[\sum_{i=1}^T \lambda(t)\right]$, where $\lambda(t)$ is a function involving the eigenvalues of both Gram matrices $\bV_{i,a(t)}(t)$ and $\widetilde{\bV}_{i,a(t)}(t)$. To proceed, we define the second-moment matrices and its \emph{effective} largest and smallest eigenvalues as
\begin{align}
\label{equ:kappa}
    \Sigma_{i,a}(t) &\triangleq \E_{\bX \sim \P_{a}}\left[ \bX_{\Pa(i)}(t)\bX_{\Pa(i)}^\top(t)\right]\ , \\
    \kappa_{\min} &\triangleq \min_{i\in[N], a\in\mcA,t\in[T]}\smin{\Sigma_{i,a}(t)} \ , \\
    \kappa_{\max} &\triangleq \max_{i\in[N], a\in\mcA,t\in[T]}\smax{\Sigma_{i,a}(t)} \ ,
\end{align}
where $\kappa_{\min}>0$ is guaranteed since there is no deterministic relation between nodes and their patients. These variables are inherent to the system and remain unknown to the learner. Given our focus on the weighted OLS estimator, we also introduce singular values related to auxiliary variables $\bX'_{\Pa(i)}(t)\triangleq\sqrt{w_i(t)}\bX_{\Pa(i)}(t)$ and $\widetilde{\bX}_{\Pa(i)}(t)\triangleq w_i(t)\bX_{\Pa(i)}(t)$. Accordingly, we define the second weighted moment matrices as follows.
\begin{align}
    \Sigma'_{i,a,w_{i}}(t)&\triangleq \E_{\bX \sim \P_{a}}\left[w_i(t) \bX_{\Pa(i)}(t) \bX_{\Pa(i)}^\top(t)\right]\ ,\\ 
    \widetilde{\Sigma}_{i,a,w_{i}}(t)&\triangleq \E_{\bX \sim \P_{a}}\left[w_i^2(t) \bX_{\Pa(i)}(t) \bX_{\Pa(i)}^\top(t)\right]\ . \label{eq:second_moment_definition_V}
\end{align} 
To bound the singular value of the weighted second moment, we first need uniform bounds for the weights. We find a bound for the norm of $\|\bX_{\Pa(i)}(t)\|_{[\widetilde{\bV}_{i,a}(t)]^{-1}}$ across all $a\in \mcA$. This yields the following result.
\begin{equation}
\|\bX_{\Pa(i)}(t)\|_{[\widetilde{\bV}_{i,a(t)}(t)]^{-1}}\leq \frac{1}{\lambda^{1/2}_{\min}(\widetilde{\bV}_{i,a(t)}(t))}\|\bX_{\Pa(i)}(t)\| \leq m \ .
\end{equation}
Then, the weights can be bounded by
\begin{equation}
     \frac{1}{C m} \leq w_i(t) = \min\left\{\frac{1}{C },\frac{1}{C \|\bX_{\Pa(i)}(t)\|_{[\widetilde{\bV}_{i,a(t)}(t)]^{-1}}}\right\} \leq \frac{1}{C } \ .
\end{equation}
Subsequently, we can bound the minimum and maximum singular values of matrices $\Sigma'_{i, a,w_{i}}(t)$ and $\widetilde{\Sigma}_{i, a,w_{i}}(t)$.
\begin{align}
    \kappa'_{\min} = \frac{1}{C m}\kappa_{\min}\leq \smin{\Sigma'_{i,a,w_{i}}(t)}\leq \smax{\Sigma'_{i,a,w_{i}}(t)}\leq \frac{1}{C }\kappa_{\max} =\kappa'_{\max}\ , \\
    \tilde{\kappa}_{\min}= \frac{1}{C^2 m^2}\kappa_{\min}\leq \smin{\widetilde \Sigma_{i,a,w_{i}}(t)}\leq \smax{\widetilde \Sigma_{i,a,w_{i}}(t)}\leq \frac{1}{C^2}\kappa_{\max} =\tilde{\kappa}_{\max}\ .
\end{align}
Moreover, we have $\norm{\bX'_{\Pa(i)}(t)} \leq m' = \frac{1}{\sqrt{C }} m$ and  $\norm{\widetilde{\bX}_{\Pa(i)}(t)} \leq \tilde{m} = \frac{1}{C} m$. In order to proceed, we need upper and lower bounds for the maximum and minimum singular values of $\bU_{i, a(t)}(t)$. However, these bounds depend on the number of non-zero rows of $\bU_{i, a(t)}(t)$ matrices, which equals to values of the random variable $N_{i, a(t)}(t)$. 
Firstly, we define the weighted constants
\begin{align}
    \gamma_n &\triangleq  \max\left\{\alpha m^2 \sqrt{n},\alpha^2 m^2 \right\} \ ,  \label{eq:def_varepsilon_n_RW}\\
    \gamma'_n &\triangleq  \max\left\{\alpha m'^2 \sqrt{n},\alpha^2 m'^2 \right\} \ ,  \label{eq:def_varepsilon_n_prime_RW}\\
    \tilde{\gamma}_n &\triangleq  \max\left\{\alpha \tilde{m}^2 \sqrt{n},\alpha^2 \tilde{m}^2 \right\} \ , \quad \forall n \in [T] \ .  \label{eq:def_varepsilon_n_tilde_RW}
\end{align}
Then for every $i\in [N], t\in [T]$, and $n\in[t]$, we define the error events corresponding to the maximum and minimum singular values of $\bU_{i}(t)$ and $\widetilde{\bU}_{i}(t)$ as
\begin{align}
    \mcE_{i,n}(t) \triangleq& \Bigg\{ N_{i}(t) = n \quad \text{and} \notag\left\{\smin{\bU_{i}(t)}\leq \sqrt{\max \left \{0, n \kappa'_{\min}- \gamma'_n\right\}}\right. \\ 
    & \hspace{-0.in}  \  \text{or} \  \left. \smax{\bU_{i}(t)}\geq \sqrt{n \kappa'_{\max} + \gamma'_n}  \right\} \Bigg\} \ , \label{eq:def_error_int_obs_RW} \\
    \mcE^{*}_{i,n}(t) \triangleq& \Bigg\{ N^{*}_{i}(t) = n \quad \text{and} \notag  \left\{\smin{\bU^{*}_{i}(t)}\leq \sqrt{\max \left \{0, n \kappa'_{\min}- \gamma'_n\right\}}\right.\\ 
    & \hspace{-0.in} \  \text{or} \   \left.\smax{\bU^{*}_{i}(t)}\geq \sqrt{n \kappa'_{\max}+\gamma'_n} \right\} \Bigg\}  \ , \label{eq:def_error_int_int_RW} \\
    \widetilde{\mcE}_{i,n}(t) \triangleq& \Bigg\{ N_{i}(t) = n \quad \text{and} \notag\left\{\smin{\widetilde{\bU}_{i}(t)}\leq \sqrt{ \max\left\{0, n \tilde{\kappa}_{\min} - \tilde{\gamma}_n\right\}}\right. \\ 
    & \hspace{-0.in}  \  \text{or} \  \left. \smax{\widetilde{\bU}_{i}(t)}\geq \sqrt{n \tilde{\kappa}_{\max} + \tilde{\gamma}_n}  \right\} \Bigg\} \ , \label{eq:def_error_int_obs_tilde_RW} \\
    \widetilde{\mcE}^{*}_{i,n}(t) \triangleq& \Bigg\{ N^{*}_{i}(t) = n \quad \text{and} \notag  \left\{\smin{\widetilde{\bU}^{*}_{i}(t)}\leq \sqrt{\max\left\{0, n \tilde{\kappa}_{\min}- \tilde{\gamma}_n\right\}}\right.\\ 
    & \hspace{-0.in} \  \text{or} \   \left.\smax{\widetilde{\bU}^{*}_{i}(t)}\geq \sqrt{n \tilde{\kappa}_{\max}+\tilde{\gamma}_n} \right\} \Bigg\}  \ . \label{eq:def_error_int_int_tilde_RW}
\end{align}
In other words, the event $\mcE_{i,n}(t)$ indicates the situation in which either $\smin{\bU_{i}(t)}$ or $\smax{\bU_{i}(t)}$ violates the established lower and upper bounds. Likewise, $\mcE^{*}_{i,n}(t),\widetilde{\mcE}_{i,n}(t)$ and $\widetilde{\mcE}^{*}_{i,n}(t)$  are associated with the singular values of $\bU^{*}_{i}(t),\widetilde{\bU}_{i}(t)$ and $\widetilde{\bU}^{*}_{i}(t)$. The next result shows that these events occur with a low probability.

\begin{lemma} \label{lm:error_event_int_RW}
The probability of the error events $\mcE_{i,n}(t)$, $\mcE^{*}_{i,n}(t)$, $\widetilde{\mcE}_{i,n}(t)$ and $\widetilde{\mcE}^{*}_{i,n}(t)$ defined in \eqref{eq:def_error_int_obs_RW} to \eqref{eq:def_error_int_int_tilde_RW} are upper bounded as
\begin{align}
    \P(\mcE_{i,n}(t)) &\leq d \exp \left( -\frac{3\alpha^2}{16} \right) \ , \qquad
    \P(\mcE^{*}_{i,n}(t)) \leq d \exp \left( -\frac{3\alpha^2}{16} \right) \ , \\
    \mbox{and} \qquad\P(\widetilde{\mcE}_{i,n}(t)) &\leq d \exp \left( -\frac{3\alpha^2}{16} \right) \ , 
     \qquad \P(\widetilde{\mcE}^{*}_{i,n}(t)) \leq d \exp \left( -\frac{3\alpha^2}{16} \right) \ .
\end{align}
\end{lemma}
\begin{proof}
    This is a direct result from \cite[Lemma~3]{varici2022causal}, where we use $\kappa'_{\min}, \kappa'_{\max}$ and $\gamma'$ for events $\mcE_{i,n}(t)$ and $\mcE^{*}_{i,n}(t)$ and using $\tilde{\kappa}_{\min}, \tilde{\kappa}_{\max}$ and $\tilde{\gamma}$ for events $\widetilde{\mcE}_{i,n}(t)$ and $\widetilde{\mcE}^{*}_{i,n}(t)$. The only difference is that one need to use the following lemma, which is an immediate result of \cite[Lemma~6]{varici2022causal}.
    \begin{lemma}\label{lm:smin_simple_bound}
    Consider matrices $\bU$ and $\bA$ that satisfy 
    \begin{align}\label{eq:smin_simple_bound_condition}
        \norm{\bU^\top \bU - \bA} \leq \gamma \ . 
    \end{align}
    Then we have,
    \begin{align}
        \smax{\bU} &\leq \sqrt{\smax{\bA}+\gamma} \ , \quad 
        \mbox{and} \quad \smin{\bU} \geq \sqrt{\max\{0, \smin{\bA}- \gamma\}}     \ .
    \end{align}
    \end{lemma}

Now that we have bounds on the probability of error events, we define the union error event $\mcE_{\cup}$ as 
\begin{equation} \label{eq:def_union_error_singular_RW}
    \mcE_{\cup} \triangleq \{ \exists\ (i,t,n) : i \in [N], t \in [T], n \in [t], \ \mcE_{i,n}(t) \ \text{or} \  \mcE^{*}_{i,n}(t) \ \text{or} \ \widetilde{\mcE}_{i,n}(t) \ \text{or} \  \widetilde{\mcE}^{*}_{i,n}(t)  \} \ .
\end{equation}
By taking a union bound and using Lemma \ref{lm:error_event_int_RW} we have
\begin{align}
    \P(\mcE_{\cup}) &\leq \sum_{i=1}^{N} \sum_{t=1}^{T} \sum_{n=1}^{t} \left(\P(\mcE_{i,n}(t)) + \P(\mcE^{*}_{i,n}(t)) + \P(\widetilde{\mcE}_{i,n}(t)) + \P(\widetilde{\mcE}^{*}_{i,n}(t))\right) \\
    &\leq 2 N T (T+1) d \exp \left( -\frac{3\alpha^2}{16} \right) \ . \label{eq:def_union_error_singular_prob_RW}
\end{align}

\paragraph{Bounding $\E \left[\mathbbm{1}\{\mcE_{\cup}\} \sum_{t=1}^T \lambda(t) \right]$} 
Since $\lmin{\bV_{i,a(t)}(t)}\geq 1$, we have the following unconditional upper bound.
\begin{align}
    \lambda(t) &= \max_{i \in [N]}  \frac{\sqrt{\lmax{\widetilde{\bV}_{i,a(t)}(t)}}}{\lmin{\bV_{i,a(t)}(t)  }}  \leq \max_{i \in [N]} \sqrt{\lmax{\widetilde{\bV}_{i,a(t)}(t)}}  \label{eq:big_bounding_error_1_RW}\\
    &= \max_{i \in [N]} \sqrt{ \lmax{ \sum_{s=1}^{t} \mathbbm{1}\{i \in a(s)\} w^2_i(s) \bX_{\Pa(i)}(s) \bX_{\Pa(i)}^{\top}(s)+\bI_{N}}} \\
    &\leq \max_{i \in [N]} \sqrt{ 1 + \frac{1}{C ^2}\sum_{s=1}^{t} \mathbbm{1}\{i \in a(s)\} \lmax{\bX_{\Pa(i)}(s) \bX_{\Pa(i)}^{\top}(s)} }\\
    &\leq \max_{i \in [N]} \sqrt{ 1 + \frac{1}{C ^2} \sum_{s=1}^{t} \lmax{\bX_{\Pa(i)}(s) \bX_{\Pa(i)}^{\top}(s)} } \\
    &= \max_{i \in [N]}  \sqrt{ 1 + \frac{1}{C ^2} \sum_{s=1}^{t} \norm{\bX_{\Pa(i)}(s)}^2 }\\
    &\leq\sqrt{ \frac{m^2}{C^2}t + 1 }\ ,  \label{eq:big_bounding_error_2_RW}
\end{align}
where~\eqref{eq:big_bounding_error_2_RW} follows from the fact that $\norm{\bX}\leq m$. Hence, we have
\begin{align}
    \E \left[\mathbbm{1}\{\mcE_{\cup}\} \sum_{t=1}^T \lambda(t) \right] &\overset{\eqref{eq:big_bounding_error_1_RW}}{\leq} \E \left[\mathbbm{1}\{\mcE_{\cup}\} \sum_{t=1}^T \sqrt{\lmax{\widetilde{\bV}_{N,a(t)}(t)}}  \right] \\
    &\overset{\eqref{eq:big_bounding_error_2_RW}}{\leq} \E \left[\mathbbm{1}\{\mcE_{\cup}\} \sum_{t=1}^T \sqrt{ \frac{m^2}{C ^2}t + 1}  \right] \\
    &= \E[\mathbbm{1}\{\mcE_{\cup}\}] \sum_{t=1}^T \sqrt{ \frac{m^2}{C ^2}t + 1}   \\
    &= \P(\mcE_{\cup}) \underbrace{\sum_{t=1}^T \sqrt{ \frac{m^2}{C ^2}t + 1}}_{A_1} \ . \label{eq:big_bounding_error_3_RW}
\end{align}
Furthermore, $A_1$ is bounded as
\begin{align}
    \sum_{t=1}^T \sqrt{ \frac{m^2}{C ^2}t + 1} &\leq \frac{m}{C }\sqrt{T}+1 +  \sum_{t=1}^{T-1} (\frac{m}{C }\sqrt{t}+1) \\
    &\leq \frac{m}{C }\sqrt{T} + T + \int_{t=1}^{T} \frac{m}{C }\sqrt{t} dt \\
    &= \frac{m}{C }\sqrt{T} + T + \frac{2m}{3C }(T^{3/2}-1) \ . \label{eq:big_bounding_error_4_RW}
\end{align}
By setting $\alpha = \sqrt{\frac{16}{3}\log(2d N T^{5/2}(T+1))}$, we obtain
\begin{align}
     \E \left[\mathbbm{1}\{\mcE_{\cup}\} \sum_{t=1}^T \sqrt{m^2 t + 1}  \right] &\overset{\eqref{eq:big_bounding_error_3_RW}}{\leq} \P(\mcE_{\cup}) \sum_{t=1}^{T}\sqrt{m^2t+1} \\
     &\overset{\eqref{eq:def_union_error_singular_prob_RW}}{\leq} \underset{= T^{-3/2}}{\underbrace{\frac{N T (T+1) d}{\exp(\log(d N T^{5/2} (T+1)))}}}  \sum_{t=1}^{T}\sqrt{m^2t+1}  \\
     &\overset{\eqref{eq:big_bounding_error_4_RW}}{\leq} T^{-3/2} \left(\frac{m}{C }\sqrt{T} + T+ \frac{2m}{3C }(T^{3/2}-1)\right) \\
     & < \frac{m}{C T} + \frac{2m}{3C } +   1 \ . \label{eq:error_event_bound_RW}
\end{align}

\paragraph{Bounding $\E \left[\mathbbm{1}\{\mcE_{\cup}^\C\} \sum_{t=1}^T \lambda(t)\right]$} 
Considering the event $\mcE_{\cup}^\C$, which encompasses all the events $\{\mcE_{i,n}^{\C}(t),$ ${\mcE^{*}_{i,n}}^{\C}(t), \widetilde{\mcE}_{i,n}^{\C}(t), \widetilde\mcE^{*}_{i,n}{}^{\C}(t) : i \in [N], t \in [T], n \in [t]\}$ being hold. Therefore, we can use the following bounds on the singular values
\begin{align}
    \smin{\bU_{i,a(t)}(t)} &\geq  \sqrt{\max\left\{0, N_{i,a(t)}(t)\kappa'_{\min}- \gamma'_n\right\}}  \label{eq:smin_D_ita_RW} \ ,  \\
    \smax{\widetilde{\bU}_{i,a(t)}(t)} &\leq \sqrt{N_{i,a(t)}(t) \tilde{\kappa}_{\max}+ \tilde{\gamma}_n}\label{eq:smax_D_ita_RW}  \ .
\end{align}
Thus, the targeted sum can be upper-bounded by
\begin{align}
    \E \left[\mathbbm{1}\{\mcE_{\cup}^\C\} \sum_{t=1}^T\lambda(t)\right] 
    &= \E \left[\mathbbm{1}\{\mcE_{\cup}^\C\} \sum_{t=1}^T \max_{i \in [N]}  \frac{\sqrt{\lmax{\widetilde{\bV}_{i,a(t)}(t)}}}{\lmin{\bV_{i,a(t)}(t)  }}\right]  \\
    &= \E \left[\mathbbm{1}\{\mcE_{\cup}^\C\} \sum_{t=1}^T \max_{i \in [N]}  \frac{\sqrt{\smaxx{\widetilde{\bU}_{i,a(t)}(t)}{2}+1}}{\sminn{\bU_{i,a(t)}(t)}{2}+1}\right] \\
    &\leq \E  \sum_{t=1}^T \max_{i \in [N]} \left[ \frac{\sqrt{N_{i,a(t)}(t) \tilde{\kappa}_{\max}+ \tilde{\gamma}_n+1} }{\max\left\{0, N_{i,a(t)}(t)\kappa'_{\min}- \gamma'_n\right\}+1} \right] \ .
\end{align}
It is worth noting that the term in the summation has a critical point, and we bound the two regions separately. To initiate this process, we define the function $h(x)$ as
\begin{equation}
    h(x) \triangleq \frac{\sqrt{x \tilde{\kappa}_{\max}+ \tilde{\gamma}_n+1} }{\max\left\{0, x\kappa'_{\min}- \gamma'_n\right\}+1}\ , \quad x>0  \label{eq:h_function_RW} \ .
\end{equation}
In order to analyze the behavior of the function $h$, we introduce $\tau\triangleq\frac{\alpha^2m^6}{\kappa_{\min}^2}$ as the critical point. Note that when $x\leq \tau$, we have $x\kappa'_{\min} < \gamma_n$. In this case, $h(x)$ is equal to
\begin{equation}
    h(x) = \sqrt{x \tilde{\kappa}_{\max}+ \tilde{\gamma}_n+1}  \label{eq:h_function_small} \ ,
\end{equation}
which is an increasing function over the region. To upper bound the $h$ function when $x>\tau$, we define the  $g$ function when $x>\tau$ as follows.
\begin{equation}
    g(x) \triangleq  \frac{\sqrt{x \kappa_{\max}+ \alpha m^2\sqrt{x}}}{ x\kappa_{\min}/m- \alpha m^2\sqrt{x}} + \frac{C }{ x\kappa_{\min}/m- \alpha m^2\sqrt{x}}\ , \quad x\geq \tau \ .
\end{equation}
Then when $x > \tau$, we establish the following relation.
\begin{align}
    h(x)= \frac{\sqrt{x \tilde{\kappa}_{\max}+ \tilde{\gamma}_n+1}}{ x\kappa'_{\min}- \gamma'_n+1} =\frac{\sqrt{x \kappa_{\max}+ \gamma_n+C ^2}}{ x\kappa_{\min}/m- \gamma_n+C }\label{equ:boundh:1}  < \frac{\sqrt{x \kappa_{\max}+ \gamma_n}}{ x\kappa_{\min}/m- \gamma_n} \ = g(x) \ ,
\end{align}
where the equality holds due to the definitions in \eqref{eq:def_varepsilon_n_RW}-\eqref{eq:def_varepsilon_n_tilde_RW} and the inequality holds due to the fact that $\frac{\sqrt{x+a^2}}{y+a} < \frac{\sqrt{x}}{y} + \frac{a}{y}$ when $a>0$. Moreover, when $x>\tau$, we have $\gamma_n = \alpha m^2\sqrt{n}$. Based on this, we can use the $g$ function to upper bound $h$ function when $x> \tau$. However, since $g$ tends to infinity at $\tau$, we begin bounding from a larger constant, specifically $4\tau$. However, for the summation when bounding $x\leq 4\tau$, we need to consider the following monotonicity.

\begin{lemma}
\label{lem:gh}
$h(x)$ and $g(x)$ are both decreasing functions when $x>\tau$, where $\tau$ is defined as $\frac{\alpha^2m^6}{\kappa_{\min}^2}$.
\end{lemma}
\begin{proof}
See Appendix~\ref{proof:lem:gh}.
\end{proof}
Now we are ready to bound the last term
\begin{equation}
     \E \left[\mathbbm{1}\{\mcE_{\cup}^\C\} \sum_{t=1}^T\lambda(t)\right]\leq \E \sum_{t=1}^{T} \max_{i\in[N]} h(N_{i,a(t)}(t))\ .
\end{equation}
We define the set of time indices at which the chosen actions are under-explored as
\begin{equation}
    \mcK_h \triangleq \left\{ t\in[T]~|~ \exists~i\in[N]: N_{i,a(t)}(t)\leq 4\tau  \right\}\ .
\end{equation}
It can be readily verified that $|\mcK_h|\leq 8N\tau$. Furthermore, when $x\in\mcK_h$, we have
\begin{align}
    h(x) \leq h(\tau) \leq \frac{1}{C }\sqrt{\kappa_{\max}\tau+\alpha m^2 \sqrt{\tau}} +1 \ , \ x \leq \tau \ .
\end{align}
Then we can upper bound the summation when $\mcK_h$ occurs as follows.
\begin{equation}
    \E \bigg[ \sum_{t=1}^T \mathbbm{1}\{t\in \mcK_h\} \max_{i \in [N]}  h(N_{i,a(t)}(t)) \bigg] \leq 2\sum_{i=1}^{N} \sum_{n=1}^{4\tau} h(n) \leq 8N \tau \left(\frac{1}{C }\sqrt{\kappa_{\max}\tau+\alpha m^2 \sqrt{\tau}} +1\right) \ .
\end{equation}
Now we only need to bound the remaining part when $t \not\in \mcK_h$ 
\begin{equation}
    \E \sum_{t=1}^T \mathbbm{1}\{t\in \mcK_h^{\C}\} \max_{i \in [N]}  h(N_{i,a(t)}(t)) \ .
\end{equation}
Note that when $t\in \mcK_h^{\C}$, we have $N_{i,a(t)}(t)>\tau$ for all $i\in[N]$, and we have 
\begin{equation}
    \max_{i\in [N]} h(N_{i,a(t)}(t))\leq   \max_{i\in [N]} g(N_{i,a(t)}(t)) \ .
\end{equation}
Furthermore, note that there might be multiple nodes that achieve $\max_{i \in [N]} g(N_{i,a(t)}(t))$. Without loss of generality, we select the solution with minimum index as $\argmax$ (or $\argmin$).
We denote the node that achieves the maximum value of the function $g$ as $i_t$.
\begin{align}
    i_t &\triangleq \argmax_{i \in [N]} g(N_{i,a(t)}(t)) = \argmin_{i \in [N]} N_{i,a(t)}(t) \ ,
\end{align}
where the last equality is a result of the fact that $g$ is a decreasing function when $x\geq \tau$.
Note that $i_t$ does not capture whether $i$ belongs to $a(t)$ or not. To address this challenge, we define the sets of time indices where $i_t=i$ for each of these two cases as follows.
\begin{align}
    \mcS_{i} &\triangleq \{t \in [T] : t\not\in\mcK_h, \ i_t = i, \ i \notin a(t) \} \ , \quad \forall i \in [N] \ , \\
    \mbox{and} \qquad \mcS^{*}_i &\triangleq \{t \in [T] : t\not\in\mcK_h, \  i_t = i, \ i \in a(t) \} \ , \quad \forall i \in [N] \ . 
\end{align}
Denote the elements of $\mcS_{i}$ by $S_{i,1},\dots, S_{i,|\mcS_{i}|}$. Until time $S_{i,n}$, the event $\{i_t=i, i \notin a(t)\}$ occurs exactly $n$ times outside $\mcK_h$ set. Similarly $\{i_t=i, i \in a(t)\}$ occurs $n$ times outside $\mcK_h$ set until time $S^{*}_{i,n}$. Then
\begin{align}
    N_{i}(S_{i,n}) = \sum_{t=1}^{S_{i,n}} \mathbbm{1}\{i \notin a(t)\} & \geq \sum_{t=1}^{S_{i,n}} \mathbbm{1}\{i_t=i, i \notin a(t)\}=n+4\tau  \ , \label{eq:h_side_2} \\
    \mbox{and} \qquad N^{*}_{i}(S^{*}_{i,n}) = \sum_{t=1}^{S^{*}_{i,n}} \mathbbm{1}\{i \in a(t)\}  &\geq \sum_{t=1}^{S^{*}_{i,n}} \mathbbm{1}\{i_t=i, i \in a(t)\}=n+4\tau  \ . \label{eq:h_side_3}
\end{align}
Using the above results and noting that $g$ is a decreasing function, we obtain
\begin{align}
    \sum_{t=1}^{T} \mathbbm{1}\{t\in \mcK_h^{\C}\}  \max_{i \in [N]}g(N_{i_t,a(t)}(t)) &= \sum_{t=1}^{T} \mathbbm{1}\{t\in \mcK_h^{\C}\} g(N_{i_t,a(t)}(t)) \\
    &=\sum_{i=1}^{N} \sum_{t: t \in \mcS_{i}} g(N_{i}(t)) + \sum_{i=1}^{N} \sum_{t: t \in \mcS^{*}_i} g(N_{i}^{*}(t)) \\ 
    &= \sum_{i=1}^{N} \sum_{n=1}^{|\mcS_{i}|} \underset{\overset{\eqref{eq:h_side_2}}{\leq} g(n+4\tau)}{\underbrace{g(N_{i}(S_{i,n}))}} + \sum_{i=1}^{N} \sum_{n=1}^{|\mcS^{*}_i|} \underset{\overset{\eqref{eq:h_side_3}}{\leq} g(n+4\tau)}{\underbrace{g(N^{*}_{i}(S^{*}_{i,n}))}} \\
    &\leq \sum_{i=1}^{N} \sum_{n=4\tau+1}^{|\mcS_{i}|+4\tau} g(n) + \sum_{i=1}^{N} \sum_{n=4\tau+1}^{|\mcS^{*}_i|+4\tau} g(n)\label{equ:G_bound_e:final}  \ .
\end{align}
We bound the discrete sums through integrals and define
\begin{equation}
    G_{\tau}(y) = \int_{x=4\tau}^{y} g(x)dx \ , \quad y\geq 4\tau \ . \label{eq:Hx_definition}
\end{equation}
Since $g(x)$ is a positive, non-increasing function, for any $k \in \mathbb{N}, k\geq 4\tau+1$ we have
\begin{align}
    \sum_{n=4\tau+1}^{k} g(n) \leq \int_{x=4\tau}^{k} g(x)dx = G_{\tau}(k) \label{eq:h_side_4} \ .
\end{align}
Then, the summation is upper bounded by
\begin{align}
    \sum_{t=1}^{T}\mathbbm{1}\{t\in \mcK_h^{\C}\}  \max_{i \in [N]}g(N_{i_t,a(t)}(t)) &\leq \sum_{i=1}^{N} \sum_{n=4\tau+1}^{|\mcS_{i}|+4\tau} g(n) + \sum_{i=1}^{N} \sum_{n=4\tau+1}^{|\mcS^{*}_i|+4\tau} g(n) \\ &\overset{\eqref{eq:h_side_4}}{\leq}  \sum_{i=1}^{N} G_{\tau}(|\mcS_{i}|+4\tau) + \sum_{i=1}^{N} G_{\tau}(| \mcS^{*}_{i}|+4\tau) \ . \label{eq:h_side_5}
\end{align}
Since $g(x)$ is positive and decreasing, and $G(y)$ is defined as an integral of the $g$ function with a positive first derivative and negative second derivative, it can be deduced that $G$ is a concave function.
\begin{align}
    \sum_{i=1}^{N} G_{\tau}(|\mcS_{i}|+4\tau) + \sum_{i=1}^{N} G_{\tau}(|\mcS^{*}_{i}|+4\tau) &\leq 2N \times G_{\tau} \left(\frac{1}{2N} \sum_{i=1}^{N} |\mcS_{i}| +  \frac{1}{2N} \sum_{i=1}^{N} |\mcS^{*}_{i}|+4\tau \right) \\
    &\leq 2N \times G_{\tau}\left(\frac{T}{2N}+4\tau \right) \ . \label{eq:h_side_6}
\end{align}
Next, we proceed to establish an upper bound for the function $G$.
\begin{lemma} \label{lem:G}
    The $G$ function can be upper bounded as
    \begin{align}
    &G_{\tau}\left(\frac{T}{2N}+4\tau\right) \nonumber\\
    &\leq 2\frac{\sqrt{m \kappa_{\max}}}{\kappa_{\min}}\left(\sqrt{\frac{T}{2N}}+\sqrt{\tau} \log\left(\sqrt{\frac{T}{2N}} + \sqrt{\tau}\right)\right)\nonumber\\
    &\quad+ \frac{4}{\kappa_{\min}}\sqrt[4]{\frac{T}{2N}} + 2\sqrt{\frac{\alpha m^5}{\kappa_{\min}^3}} \log \left(\frac{\sqrt{\frac{1}{\tau}}\sqrt[4]{\frac{T}{2N}}+\sqrt[4]{4}+1}{\sqrt{\frac{1}{\tau}}\sqrt[4]{\frac{T}{2N}}+\sqrt[4]{4}-1}\right)\nonumber\\
    &\quad + \frac{2m\log\left(\frac{\kappa_{\min}}{m}\sqrt{\frac{T}{2N}}+\alpha m^2\right)}{\kappa_{\min}} C  \ . \label{equ:G_bound_ec:final_lemma}
    \end{align}
\end{lemma}

\begin{proof}
    See Appendix~\ref{proof:lem:G}
\end{proof}

Combining the results in \eqref{eq:error_event_bound_RW}, \eqref{equ:G_bound_e:final}, \eqref{eq:h_side_6} and \eqref{equ:G_bound_ec:final_lemma}, let $E_1$ denote the accumulation of terms that exhibit at most logarithmic growth rates with respect to $T$ and $C $.
\begin{align}
    E_1&= 4N\frac{\sqrt{m \kappa_{\max}}}{\kappa_{\min}} \sqrt{\tau} \log\left(\sqrt{\frac{T}{2N}} + \sqrt{\tau}\right) + 4N\sqrt{\frac{\alpha m^5}{\kappa_{\min}^3}} \log \left(\frac{\sqrt{\frac{1}{\tau}}\sqrt[4]{\frac{T}{2N}}+\sqrt[4]{4}+1}{\sqrt{\frac{1}{\tau}}\sqrt[4]{\frac{T}{2N}}+\sqrt[4]{4}-1}\right)\\
    &\quad + 8N \tau \left(\frac{1}{C }\sqrt{\kappa_{\max}\tau+\alpha m^2 \sqrt{\tau}} +1\right) + \frac{m}{C T} + \frac{2m}{3C } +   1 \ .
\end{align}
Therefore, the final result for the bound is
\begin{align}
    \E\left[\sum_{t=1}^T  \lambda(t) 
      \right] &\leq \frac{4\sqrt{m \kappa_{\max}}}{\kappa_{\min}}\sqrt{NT} + \frac{8}{\kappa_{\min}}\sqrt[4]{\frac{N^3T}{2}} \\
      &\quad + \frac{4Nm}{\kappa_{\min}}\log\left(\frac{\kappa_{\min}}{m}\sqrt{\frac{T}{2N}}+\alpha m^2\right) C  +E_1 \ .
      \label{eq:accumulation_E1}
\end{align}
Plugging~\eqref{eq:accumulation_E1} into~\eqref{eq:ucb_proof_mid11_RW}, we have
\begin{equation}
    \E\left[R(T)\right]\leq 4 (\beta_T+1)^{L} d^{\frac{L-1}{2}}\E\left[\sum_{t=1}^T  \lambda(t) 
      \right]  = \tilde{\mcO}(d^{L-\frac{1}{2}}\sqrt{NT}+d^{L-\frac{1}{2}} N C )\ .
\end{equation}
\end{proof}
\section{Discussion on \algonameLinSEMUCB}
\begin{lemma}
\label{lem:naive}
By setting the robust confidence radius 
\begin{equation}
\label{equ:betaprime}
    \beta'_t(\delta)=1+C m^2+ \sqrt{2\log(\delta)+d\log(1+m^2t/d)} \ .
\end{equation}
Then with probability at least $1-2\delta$, for any node $i\in[N]$ and $t\in\N$, the OLS estimator satisfies
\begin{equation}
    [\bB]_i\in\mcC_{i}(t) \ , \quad \text{and} \quad [\bB^*]_i\in\mcC^*_{i}(t) \ .
\end{equation}
\end{lemma}
\begin{proof}
\label{proof:lem:naive}
We will provide the proof corresponding to the observational weights $|[\bB(t)]_i$, while the proof for the interventional weights $[\bB^*(t)]_i$ follows similarly. For any node $i\in[N]$, let us decompose the error in estimation $\|[\bB(t)]_i-\bB_i\|_{\bV_{i}(t)}$ as follows.
\begin{align}
     &\|[\bB(t)]_i-\bB_i\|_{\bV_{i}(t)}\\
     &=\biggl\|[\bV_{i}(t)]^{-1}  \!\!\!\!\!\!\! \sum_{s\in[t-1],i \notin a(t)}  \!\!\!\!\!\!\! \bX_{\Pa(i)}(s) \big[\bX^{\top}_{\Pa(i)}(s)\left([\bB]_i+[\Delta(s)]_i\right)+\epsilon_{i}(s)-\nu_i\big]-[\bB]_i\biggl\|_{\bV_{i}(t)}\\
    &\leq\underbrace{\biggl\| \sum_{s\in[t-1],i\notin a(t)}\!\!  \bX_{\Pa(i)}(s)(\epsilon_{i}(s)-\nu_i)\biggl\|_{[\bV_{i}(t)]^{-1}}}_{I_1\text{: Stochastic error}} \\
    &\quad+\underbrace{\biggl\|\sum_{s\in[t-1],i\notin a(t)}\!\!  \bX_{\Pa(i)}(s) \bX^{\top}_{\Pa(i)}(s)[\Delta(s)]_i\biggl\|_{[\bV_{i}(t)]^{-1}}}_{I_2\text{: Fluctuation error}} +\underbrace{\| [\bB]_i\|_{[\bV_{i}(t)]^{-1}}}_{I_3\text{: Regularization error}}\label{equ:errorDecomposition}~.
\end{align}
Similarly to~\cite[Theorem 20.5]{lattimore2020bandit}, with probability at least $1-\delta$ the sum of regularization error and the stochastic error  $I_1+I_3$ can be bounded by $1+\sqrt{2\log(\delta)+d\log(1+m^2T/d)}$. However, the fluctuation error $I_2$ arises due to model deviation, and it needs to be upper-bounded to obtain the confidence sequence for the OLS estimator. Similar to the existing literature on robust linear bandits~\cite{zhao2021linear,ding2022robust}, one can bound this term by triangle inequality and Cauchy-Schwarz inequality, as follows.
\begin{align}
    I_2 &\leq  \sum_{s\in[t-1],i\notin a(t)}\norm{ \bV_{i}(t)^{-1/2} \bX_{\Pa(i)}(s) \bX^{\top}_{\Pa(i)}(s)[\Delta_\bB(s)]_i}\nonumber\\
    &\leq m \sum_{s\in[t-1],i\notin a(t)} \| \bX_{\Pa(i)}(s)\|_{[\bV_{i}(t)]^{-1}} \norm{[\Delta_\bB(s)]_i}\nonumber\\
    &\leq C m^2\ .
\end{align}
Thus we conclude that with probability $1-\delta$
\begin{equation}
    \|[\bB(t)]_i-\bB_i\|_{\bV_{i}(t)}\leq \beta'_t(\delta) \ .
\end{equation}
\end{proof}
Then we have the following theorem for the regret of \algonameLinSEMUCB\ under model deviation.
\begin{theorem}
Under a deviation budget $C$, by setting $\delta=\frac{1}{2NT}$ and $\beta'_t(\delta)$ according to \eqref{equ:betaprime}, the average cumulative regret of \algonameLinSEMUCB is upper bounded by
\begin{equation}
     \E[R(T)]\leq 2m+\tilde{\mcO}(\beta_T^{\prime L} d^{\frac{L}{2}}\sqrt{NT})\ ,
\end{equation}
where $\beta_T^{\prime} = 1+C m^2+ \sqrt{2\log(2NT)+d\log(1+m^2t/d)}$
\end{theorem}
\begin{proof}
    This theorem directly follows \cite[Theorem 2]{varici2022causal}
    by changing the confidence radius.
\end{proof}

\section{Proof of Lemma~\ref{lem:gh}}
\label{proof:lem:gh}
We first prove that $h(x)$ is a decreasing function. We start with the derivation of $h(x)$ and show that it is negative in the domain of $x\geq \tau$. To simplify the notations, let $c_1=\kappa_{\max}$, $c_2=\alpha m^2$, $c_3=\kappa_{\min}$. The first order derivative of $h(x)$ has the following closed-form.
\begin{equation}
     \frac{{\rm d}}{{\rm d}x} h(x)=\frac{(2c_2 - 4 c_3 \sqrt{x})C ^2+(c_2 +2c_1 \sqrt{x})C -2c_1 c_3 x^{3/2}+c_2^2\sqrt{x}-3c_2 c_3 x}{4\sqrt{x}\sqrt{c_1 x+c_2\sqrt{x}+C }(c_3x-c_2\sqrt{x}+C )^2} \ .
\end{equation}
Considering when $x>\tau$, the denominator of the derivative is positive, it is adequate to show the numerator is negative. We notice that the numerator is a quadratic function of $C$, and we have the following relation
\begin{equation}
    2c_2 - 4 c_3 \sqrt{x}\leq -2 c_2<0\ , \qquad c_2 +2c_1 \sqrt{x}>0\ .
\end{equation}
Thus, it is sufficient to show
\begin{align}
   &~ -\frac{(c_2+2c_1\sqrt{x})^2}{4(2c_2-3\sqrt{c_3})}-2c_1 c_3 x^{3/2} +c_2^2\sqrt{x}- 3c_2 c_3x<0 \label{equ:h_decrease:1}\\
    \Leftarrow ~&~ (c_2+2c_1\sqrt{x})^2 + 4 (-2c_1 c_3 x^{3/2} +c_2^2\sqrt{x}- 3c_2 c_3x) (3\sqrt{c_3}-2c_2)<0 \label{equ:h_decrease:2}\\
    \Leftarrow ~&~ (c_2+2c_1\sqrt{x})^2 + 8 (-2c_1 c_3 x^{3/2} -2c_2^2\sqrt{x}) c_2<0 \label{equ:h_decrease:3}\\
    \Leftarrow ~&~ (c_2^2+4 c_1c_2\sqrt{x} -16 c_2^3\sqrt{x}) + (4c_1^2 x -16c_1 c_3 x^{3/2})<0 \label{equ:h_decrease:4}\\
    \Leftarrow ~&~(2\alpha m^4-15\alpha^3m^6)\sqrt{x} + \kappa_{\max} (4m^2-16\alpha m^3)x<0\label{equ:h_decrease:5} \ ,
\end{align}
where \eqref{equ:h_decrease:3} is due to $2c_2 - 4 c_3 \sqrt{x}\geq -2 c_2>0$ and $c_2^2\sqrt{x}- 3c_2 c_3x \leq -2 c_2^2 \sqrt{x}<0$. The \eqref{equ:h_decrease:5} is true since $\alpha>2$ and the assumption $m \geq 1$.

From the above analysis, we can conclude that $h(x)$ is a decreasing function for $x>\tau$. Then we prove the $g(x)$ function is decreasing when $x>\tau$. The derivative of $g(x)$ can be calculated as
\begin{equation}
    \frac{{\rm d}}{{\rm d}x} g(x) = \frac{-2\kappa_{\max} \kappa_{\min}x/m +\alpha^2m^4 -3\alpha^2m^2 \kappa_{\min}\sqrt{x}/m}{4x\sqrt{\kappa_{\max}+\alpha m^2\sqrt{x}}(\kappa_{\min}x/m-\alpha m^2\sqrt{x})^2} - \frac{\kappa_{\min}/m - \frac{\alpha m^2}{2\sqrt{x}}}{(\kappa_{\min}x/m-\alpha m^2\sqrt{x})^2}<0 \ .
\end{equation}
We conclude that $h(x)$ and $g(x)$ are both decreasing when $x> \tau$.

\section{Proof of Lemma~\ref{lem:G}}
\label{proof:lem:G}
We start from the following results for integrals.
\begin{lemma}[Useful integral] For $a,b>0$, the following integrals hold.
\label{lem:integral}
\begin{align}
    \int \frac{1}{\sqrt{ax}-ab} {\rm d} x &= 2\frac{\sqrt{x}}{a}+2\frac{b}{a}\log(|\sqrt{x}-b|) + c \ . \\
    \int \frac{\sqrt[4]{x}}{ax-b\sqrt{x}} {\rm d} x &= \frac{4\sqrt[4]{x}}{a} - \frac{2\sqrt{b}}{a^{3/2}}\log\left(\left|\frac{1+\sqrt{\frac{a}{b}}\sqrt[4]{x}}{1-\sqrt{\frac{a}{b}}\sqrt[4]{x}}\right|\right) +c \ ,\\
    \int \frac{1}{ax-b\sqrt{x}} {\rm d} x &= \frac{2\log(|a\sqrt{x}-b|)}{a}+c \ .
\end{align} 
\end{lemma}
Then the $G$ function can be upper bounded as
\begin{align}
    G_{\tau}\left(\frac{T}{2N}+4\tau\right) &= \int_{x=4\tau}^{\frac{T}{2N}+4\tau} g(x) {\rm d}x \\
    &=\int_{x=4\tau}^{\frac{T}{2N}+4\tau} \frac{\sqrt{x \kappa_{\max}+ \alpha m^2\sqrt{x}}}{ x\kappa_{\min}/m- \alpha m^2\sqrt{x}} {\rm d}x + \int_{x=4\tau}^{\frac{T}{2N}+4\tau} \frac{C }{ x\kappa_{\min}/m- \alpha m^2\sqrt{x}} {\rm d}x \\
    &\leq \int_{x=4\tau}^{\frac{T}{2N}+4\tau} \sqrt{\frac{m\kappa_{\max}}{\kappa_{\min}}} \frac{1}{\sqrt{x\kappa_{\min}/m- \alpha m^2\sqrt{x}}}{\rm d}x\nonumber\\
    &\quad + \int_{x=4\tau}^{\frac{T}{2N}+4\tau}
    \sqrt{\alpha m^2(1+\frac{m\kappa_{\max}}{\kappa_{\min}})}\frac{x^{1/4}}{ x\kappa_{\min}/m- \alpha m^2\sqrt{x}} {\rm d}x \nonumber\\
    & \quad  + \int_{x=4\tau}^{\frac{T}{2N}+4\tau} \frac{C }{ x\kappa_{\min}/m- \alpha m^2\sqrt{x}} {\rm d}x \\ 
    &\leq \int_{x=4\tau}^{\frac{T}{2N}+4\tau} \sqrt{\frac{m^2\kappa_{\max}}{\kappa_{\min}}} \frac{1}{\sqrt{x\kappa_{\min}}-\sqrt{\tau\kappa_{\min}}}{\rm d}x \label{equ:G_bound:mid}\nonumber\\ 
    &\quad+ \int_{x=4\tau}^{\frac{T}{2N}+4\tau}
    \sqrt{\alpha m^2(1+\frac{m\kappa_{\max}}{\kappa_{\min}})}\frac{x^{1/4}}{ x\kappa_{\min}/m- \alpha m^2\sqrt{x}} {\rm d}x \nonumber\\
    &\quad + \int_{x=4\tau}^{\frac{T}{2N}+4\tau} \frac{C }{ x\kappa_{\min}/m- \alpha m^2\sqrt{x}} {\rm d}x \\ 
    &<2\frac{\sqrt{m \kappa_{\max}}}{\kappa_{\min}}\left(\sqrt{\frac{T}{2N}+4\tau}+\sqrt{\tau} \log\left(\sqrt{\frac{T}{2N}+4\tau} - \sqrt{\tau}\right)- 2\sqrt{\tau}\right)\nonumber\\
    &\quad + \frac{4\sqrt[4]{\frac{T}{2N}+4\tau}}{\kappa_{\min}} + 2\sqrt{\frac{\alpha m^5}{\kappa_{\min}^3}} \log \left(\frac{\sqrt{\frac{\kappa_{\min}}{\alpha m^3}}\sqrt[4]{\frac{T}{2N}+4\tau}+1}{\sqrt{\frac{\kappa_{\min}}{\alpha m^3}}\sqrt[4]{\frac{T}{2N}+4\tau}-1}\right)- \frac{4\sqrt[4]{4\tau}}{\kappa_{\min}}\nonumber\\
    &\quad + \frac{2m\log\left(\frac{\kappa_{\min}}{m}\sqrt{\frac{T}{2N}+4\tau}-\alpha m^2\right)}{\kappa_{\min}} C \label{equ:G_bound_ec:int} \\
    &\leq 2\frac{\sqrt{m \kappa_{\max}}}{\kappa_{\min}}\left(\sqrt{\frac{T}{2N}}+\sqrt{\tau} \log\left(\sqrt{\frac{T}{2N}} + \sqrt{\tau}\right)\right)\nonumber\\
    &\quad+ \frac{4}{\kappa_{\min}}\sqrt[4]{\frac{T}{2N}} + 2\sqrt{\frac{\alpha m^5}{\kappa_{\min}^3}} \log \left(\frac{\sqrt{\frac{1}{\tau}}\sqrt[4]{\frac{T}{2N}}+\sqrt[4]{4}+1}{\sqrt{\frac{1}{\tau}}\sqrt[4]{\frac{T}{2N}}+\sqrt[4]{4}-1}\right)\nonumber\\
    &\quad + \frac{2m\log\left(\frac{\kappa_{\min}}{m}\sqrt{\frac{T}{2N}}+\alpha m^2\right)}{\kappa_{\min}} C  \ . \label{equ:G_bound_ec:final} 
\end{align}
where \eqref{equ:G_bound:mid} is due to the inequality $\sqrt{x-y} \geq \sqrt{x}-\frac{y}{\sqrt{x}}$ when $x\geq y$, and we use Lemma~\ref{lem:integral} in  \eqref{equ:G_bound_ec:int}.

\bibliographystyle{IEEEbib}
\bibliography{main}

\begin{thebibliography}{10}

\bibitem{ahmed2021causalworld}
Ossama Ahmed, Frederik Tr{\"{a}}uble, Anirudh Goyal, Alexander Neitz, Manuel Wuthrich, Yoshua Bengio, Bernhard Sch{\"{o}}lkopf, and Stefan Bauer,
\newblock ``Causalworld: {A} robotic manipulation benchmark for causal structure and transfer learning,''
\newblock in {\em Proc. International Conference on Learning Representations}, virtual, May 2021.

\bibitem{badsha2019learning}
Md.~Bahadur Badsha and Audrey~Qiuyan Fu,
\newblock ``Learning causal biological networks with the principle of mendelian randomization,''
\newblock {\em Frontiers in Genetics}, vol. 10, May 2019.

\bibitem{liu2020reinforcement}
Siqi Liu, Kay~Choong See, Kee~Yuan Ngiam, Leo~Anthony Celi, Xingzhi Sun, and Mengling Feng,
\newblock ``Reinforcement learning for clinical decision support in critical care: Comprehensive review,''
\newblock {\em Journal of Medical Internet Research}, vol. 22, no. 7, July 2020.

\bibitem{zhao2022mitigating}
Yan Zhao, Mitchell Goodman, Sameer Kanase, Shenghe Xu, Yannick Kimmel, Brent Payne, Saad Khan, and Patricia Grao,
\newblock ``Mitigating targeting bias in content recommendation with causal bandits,''
\newblock in {\em Proc. {ACM} Conference on Recommender Systems Workshop on Multi-Objective Recommender Systems}, Seattle, WA, September 2022.

\bibitem{pearl2009causality}
Judea Pearl,
\newblock {\em Causality},
\newblock Cambridge University Press, Cambridge, UK, 2009.

\bibitem{zhang2021matching}
Jiaqi Zhang, Chandler Squires, and Caroline Uhler,
\newblock ``Matching a desired causal state via shift interventions,''
\newblock in {\em Proc. Advances in Neural Information Processing Systems}, virtual, December 2021.

\bibitem{lu2021causal}
Yangyi Lu, Amirhossein Meisami, and Ambuj Tewari,
\newblock ``Causal bandits with unknown graph structure,''
\newblock in {\em Proc. Advances in Neural Information Processing Systems}, virtual, December 2021.

\bibitem{lattimore2016causal}
Finnian Lattimore, Tor Lattimore, and Mark~D Reid,
\newblock ``Causal bandits: Learning good interventions via causal inference,''
\newblock in {\em Proc. Advances in Neural Information Processing Systems}, Barcelona, Spain, December 2016.

\bibitem{de2022causal}
Arnoud De~Kroon, Joris Mooij, and Danielle Belgrave,
\newblock ``Causal bandits without prior knowledge using separating sets,''
\newblock in {\em Proc. Conference on Causal Learning and Reasoning}, Eureka, CA, April 2022.

\bibitem{feng2023combinatorialwioutgraph}
Shi Feng, Nuoya Xiong, and Wei Chen,
\newblock ``Combinatorial causal bandits without graph skeleton,''
\newblock {\em arXiv:2301.13392}.

\bibitem{bilodeau2022adaptively}
Blair Bilodeau, Linbo Wang, and Daniel~M Roy,
\newblock ``Adaptively exploiting $d$-separators with causal bandits,''
\newblock in {\em Proc. Advances in Neural Information Processing Systems}, New Orleans, LA, December 2022.

\bibitem{konobeev2023causal}
Mikhail Konobeev, Jalal Etesami, and Negar Kiyavash,
\newblock ``Causal bandits without graph learning,''
\newblock {\em arXiv:2301.11401}.

\bibitem{maiti2022causal}
Aurghya Maiti, Vineet Nair, and Gaurav Sinha,
\newblock ``A causal bandit approach to learning good atomic interventions in presence of unobserved confounders,''
\newblock in {\em Proc. Conference on Uncertainty in Artificial Intelligence}, Eindhoven, Netherlands, August 2022.

\bibitem{feng2023combinatorial}
Shi Feng and Wei Chen,
\newblock ``Combinatorial causal bandits,''
\newblock in {\em Proc. the AAAI Conference on Artificial Intelligence}, Washington, D.C., 2023.

\bibitem{varici2022causal}
Burak Varici, Karthikeyan Shanmugam, Prasanna Sattigeri, and Ali Tajer,
\newblock ``Causal bandits for linear structural equation models,''
\newblock {\em The Journal of Machine Learning Research}, vol. 24, no. 297, pp. 1--59, 2023.

\bibitem{huang2017behind}
Biwei Huang, Kun Zhang, Jiji Zhang, Ruben Sanchez-Romero, Clark Glymour, and Bernhard Sch{\"o}lkopf,
\newblock ``Behind distribution shift: Mining driving forces of changes and causal arrows,''
\newblock in {\em Proc. IEEE International Conference on Data Mining}, New Orleans, LA, November 2017.

\bibitem{zhang2017causalA}
Kun Zhang, Biwei Huang, Jiji Zhang, Clark Glymour, and Bernhard Sch{\"o}lkopf,
\newblock ``Causal discovery from nonstationary/heterogeneous data: Skeleton estimation and orientation determination,''
\newblock in {\em Proc. International Joint Conferences on Artificial Intelligence}, Melbourne, Australia, August 2017.

\bibitem{zhang2017causalB}
Kun Zhang, Mingming Gong, Joseph Ramsey, Kayhan Batmanghelich, Peter Spirtes, and Clark Glymour,
\newblock ``Causal discovery in the presence of measurement error: Identifiability conditions,''
\newblock {\em arXiv:1706.03768}, 2017.

\bibitem{zhang2016identifiability}
Kun Zhang, Jiji Zhang, Biwei Huang, Bernhard Sch{\"o}lkopf, and Clark Glymour,
\newblock ``On the identifiability and estimation of functional causal models in the presence of outcome-dependent selection.,''
\newblock in {\em Proc. Conference on Uncertainty in Artificial Intelligence}, New York City, NY, June 2016.

\bibitem{tu2019causal}
Ruibo Tu, Cheng Zhang, Paul Ackermann, Karthika Mohan, Hedvig Kjellstr{\"o}m, and Kun Zhang,
\newblock ``Causal discovery in the presence of missing data,''
\newblock in {\em Proc. International Conference on Artificial Intelligence and Statistics}, Okinawa, Japan, April 2019.

\bibitem{li2017learning}
Junying Li, Deng Cai, and Xiaofei He,
\newblock ``Learning graph-level representation for drug discovery,''
\newblock {\em arXiv:1709.03741}, 2017.

\bibitem{lu2020regret}
Yangyi Lu, Amirhossein Meisami, Ambuj Tewari, and William Yan,
\newblock ``Regret analysis of bandit problems with causal background knowledge,''
\newblock in {\em Proc. Conference on Uncertainty in Artificial Intelligence}, virtual, August 2020.

\bibitem{nair2021budgeted}
Vineet Nair, Vishakha Patil, and Gaurav Sinha,
\newblock ``Budgeted and non-budgeted causal bandits,''
\newblock in {\em Proc. International Conference on Artificial Intelligence and Statistics}, virtual, April 2021.

\bibitem{sussex2022model}
Scott Sussex, Anastasiia Makarova, and Andreas Krause,
\newblock ``Model-based causal bayesian optimization,''
\newblock in {\em Proc. International Conference on Learning Representations}, Kigali, Rwanda, May 2023.

\bibitem{ghosh2017misspecified}
Avishek Ghosh, Sayak~Ray Chowdhury, and Aditya Gopalan,
\newblock ``Misspecified linear bandits,''
\newblock in {\em Proc. Conference on Artificial Intelligence}, San Francisco, CA, February 2017.

\bibitem{lattimore2020learning}
Tor Lattimore, Csaba Szepesvari, and Gellert Weisz,
\newblock ``{Learning with good feature representations in bandits and in RL with a generative model},''
\newblock in {\em Proc. International Conference on Machine Learning}, virtual, July 2020.

\bibitem{foster2020adapting}
Dylan~J Foster, Claudio Gentile, Mehryar Mohri, and Julian Zimmert,
\newblock ``Adapting to misspecification in contextual bandits,''
\newblock in {\em Proc. Advances in Neural Information Processing Systems}, virtual, December 2020.

\bibitem{krishnamurthy2021adapting}
Sanath~Kumar Krishnamurthy, Vitor Hadad, and Susan Athey,
\newblock ``Adapting to misspecification in contextual bandits with offline regression oracles,''
\newblock in {\em Proc. International Conference on Machine Learning}, Stockholm, Sweden, July 2021.

\bibitem{li2019stochastic}
Yingkai Li, Edmund~Y Lou, and Liren Shan,
\newblock ``Stochastic linear optimization with adversarial corruption,''
\newblock {\em arXiv:1909.02109}, 2019.

\bibitem{bogunovic2021stochastic}
Ilija Bogunovic, Arpan Losalka, Andreas Krause, and Jonathan Scarlett,
\newblock ``Stochastic linear bandits robust to adversarial attacks,''
\newblock in {\em Proc. International Conference on Artificial Intelligence and Statistics}, virtual, April 2021.

\bibitem{lee2021achieving}
Chung-Wei Lee, Haipeng Luo, Chen-Yu Wei, Mengxiao Zhang, and Xiaojin Zhang,
\newblock ``Achieving near instance-optimality and minimax-optimality in stochastic and adversarial linear bandits simultaneously,''
\newblock in {\em Proc. International Conference on Machine Learning}, virtual, July 2021.

\bibitem{zhao2021linear}
Heyang Zhao, Dongruo Zhou, and Quanquan Gu,
\newblock ``Linear contextual bandits with adversarial corruptions,''
\newblock {\em arXiv:2110.12615}, 2021.

\bibitem{wei2022model}
Chen-Yu Wei, Christoph Dann, and Julian Zimmert,
\newblock ``A model selection approach for corruption robust reinforcement learning,''
\newblock in {\em Proc. International Conference on Algorithmic Learning Theory}, Paris, France, March 2022.

\bibitem{he2022nearly}
Jiafan He, Dongruo Zhou, Tong Zhang, and Quanquan Gu,
\newblock ``Nearly optimal algorithms for linear contextual bandits with adversarial corruptions,''
\newblock in {\em Proc. Advances in Neural Information Processing Systems}, New Orleans, LA, December 2022.

\bibitem{russac2019weighted}
Yoan Russac, Claire Vernade, and Olivier Capp{\'e},
\newblock ``Weighted linear bandits for non-stationary environments,''
\newblock in {\em Proc. Advances in Neural Information Processing Systems}, Vancouver, British Columbia, Canada, December 2019.

\bibitem{abbasi2011improved}
Yasin Abbasi-Yadkori, D{\'a}vid P{\'a}l, and Csaba Szepesv{\'a}ri,
\newblock ``Improved algorithms for linear stochastic bandits,''
\newblock in {\em Proc. Advances in Neural Information Processing Systems}, Granada, Spain, December 2011.

\bibitem{lattimore2020bandit}
Tor Lattimore and Csaba Szepesv{\'a}ri,
\newblock {\em Bandit Algorithms},
\newblock Cambridge University Press, Cambridge, UK, 2020.

\bibitem{ding2022robust}
Qin Ding, Cho-Jui Hsieh, and James Sharpnack,
\newblock ``Robust stochastic linear contextual bandits under adversarial attacks,''
\newblock in {\em Proc. International Conference on Artificial Intelligence and Statistics}, virtual, March 2022.

\end{thebibliography}

\end{document}